\documentclass[journal]{IEEEtran}

\usepackage[utf8]{inputenc} 
\usepackage[T1]{fontenc}    
\usepackage{hyperref}       
\usepackage{url}            
\usepackage{booktabs}       
\usepackage{amsfonts}       
\usepackage{nicefrac}       
\usepackage{microtype}      
\usepackage{xcolor}         
\usepackage{pifont}
\usepackage{amsmath}
\usepackage{multirow}
\usepackage{adjustbox}
\usepackage[table,xcdraw,dvipsnames]{xcolor}  
\usepackage{tcolorbox}
\usepackage{threeparttable}
\usepackage{hyperref}
\usepackage{bm}  
\usepackage{makecell} 
\usepackage{graphicx}
\usepackage{subfigure}
\usepackage{subcaption}  
\usepackage[numbers,sort&compress]{natbib}
\usepackage{mathtools}
\usepackage{rotating}
\usepackage{wrapfig}
\usepackage{mathrsfs}
\usepackage{amsthm}
\usepackage[ruled,linesnumbered]{algorithm2e}
\definecolor{commentgreen}{RGB}{76,136,107}
\definecolor{commentblue}{RGB}{16,105,181}
\definecolor{lv}{RGB}{204, 255, 204}
\usepackage[ruled,linesnumbered]{algorithm2e}
\usepackage{bm}
\usepackage{marvosym}
\usepackage{enumitem} 
\newtheorem{definition}{Definition}

\newtheorem{proposition}{Proposition}
\setcitestyle{numbers,square}
\hypersetup{
colorlinks=true,
linkcolor=red,
citecolor=magenta,
urlcolor=magenta}

\ifCLASSINFOpdf
\else
\fi
\begin{document}
%
\title{A General ReLearner: Empowering Spatiotemporal Prediction by Re-learning Input-label Residual}
%
%
%

\author{
                Jiaming~Ma,
    		Binwu~Wang$^*$,
    		Pengkun~Wang,
    		Xu~Wang, 
            Zhengyang~Zhou,\\
    		and 
            Yang~Wang$^*$,~\IEEEmembership{Senior Member,~IEEE}

            \thanks{Prof. Yang Wang and Dr. Binwu Wang are the corresponding authors.}
    		\thanks{Jiaming Ma, Binwu Wang, Pengkun Wang, Zhengyang Zhou, Xu Wang, and Yang Wang are with University of Science and Technology of China, Heifei 230022, China (e-mail: JiamingMa@mail.ustc.edu.cn, wbw2024@ustc.edu.cn;  pengkun@ustc.edu.cn;
    			zzy0929@ustc.edu.cn; zyd2020@mail.ustc.edu.cn; 
    			wx309@ustc.edu.cn;   angyan@ustc.edu.cn).}
            \thanks{This work has been submitted to the IEEE for possible publication. Copyright may be transferred without notice, after which this version may no longer be accessible.}}

\markboth{Journal of \LaTeX\ Class Files,~Vol.~14, No.~8, August~2025}%
{Shell \MakeLowercase{\textit{et al.}}: Bare Demo of IEEEtran.cls for IEEE Journals}
%



\maketitle

\begin{abstract}
Prevailing spatiotemporal prediction models typically operate under a forward (unidirectional) learning paradigm, in which models extract spatiotemporal features from historical observation input and map them to target spatiotemporal space for future forecasting (label). However, these models frequently exhibit suboptimal performance when spatiotemporal discrepancies exist between inputs and labels, for instance, when nodes with similar time-series inputs manifest distinct future labels, or vice versa. To address this limitation, we propose explicitly incorporating label features during the training phase. Specifically, we introduce the Spatiotemporal Residual Theorem, which generalizes the conventional unidirectional spatiotemporal prediction paradigm into a bidirectional learning framework. Building upon this theoretical foundation, we design an universal module, termed ReLearner, which seamlessly augments Spatiotemporal Neural Networks (STNNs) with a bidirectional learning capability via an auxiliary inverse learning process. In this process, the model relearns the spatiotemporal feature residuals between input data and future data. The proposed ReLearner comprises two critical components: (1) a Residual Learning Module, designed to effectively disentangle spatiotemporal feature discrepancies between input and label representations; and (2) a Residual Smoothing Module, employed to smooth residual terms and facilitate stable convergence. Extensive experiments conducted on 11 real-world datasets across 14 backbone models demonstrate that ReLearner significantly enhances the predictive performance of existing STNNs. Our code is available on \href{https://anonymous.4open.science/r/nips2025427}{GitHub}.
\end{abstract}

%
\IEEEpeerreviewmaketitle

\section{Introduction}

\IEEEPARstart{S}{patiotemporal} prediction plays a vital role in a wide range of domains, such as intelligent transportation systems~\cite{zheng2014urban, ma2025less, zheng2013u,wu2019graph, yu2017spatio}, environmental management~\cite{ma2025causal, ma2025spatiotemporal, wang2020pm2}, public health~\cite{ma2025mofo}, and energy demand management~\cite{ma2025mobimixer}. In these contexts, data are typically modeled as spatiotemporal graphs, with nodes representing sensors, monitoring stations, or geographic regions and edges reflecting the weighted spatial relationships among them. The core objective of spatiotemporal prediction is to estimate the subsequent future values (label) by analyzing the underlying spatiotemporal patterns in the input historical observations of all nodes~\cite{shao2022decoupled, jin2023survey}.

\begin{figure*}[!t] 
	\centering
	\includegraphics[width=0.9\textwidth]{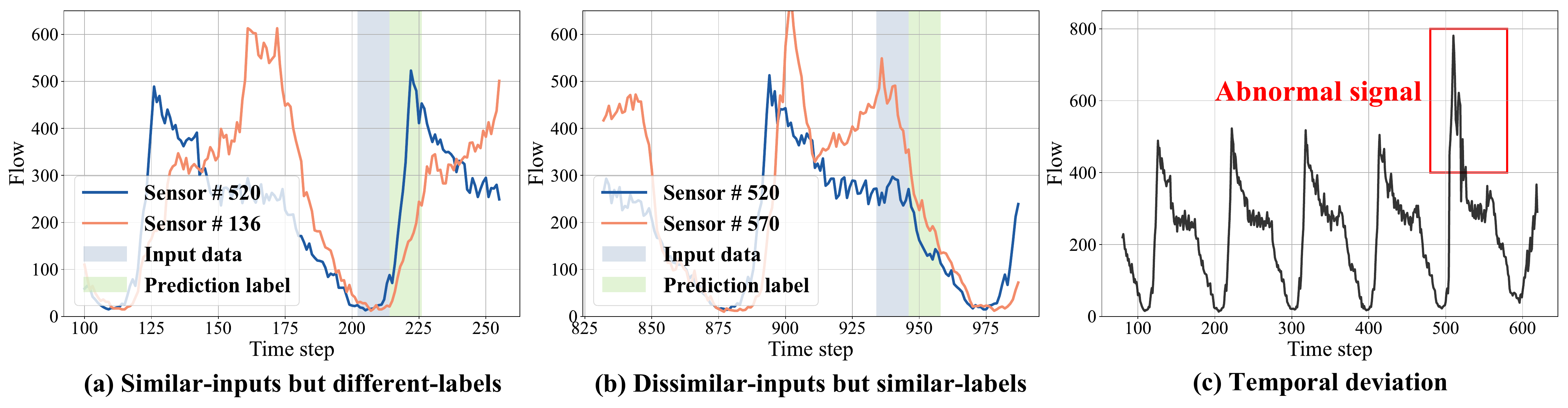}
	\caption{Visualization examples of input-label deviation on LargeST-SD dataset: (a, left) similar histories produce divergent lables; (b, middle) different inputs yield similar lables; and (c, right) abrupt, atypical temporal shifts.}
	\label{fig0}
\end{figure*}

The classic paradigm for achieving accurate spatiotemporal prediction follows a forward process: it leverages meticulously designed spatiotemporal learning methods to capture the spatiotemporal dependencies within the input historical data, and subsequently maps this to the label space to generate predictions. Among these methods, Spatiotemporal Neural Networks (STNNs) have emerged as the most prevalent architecture, effectively combining spatial modeling with sequence-aware temporal learning. The temporal module captures evolutionary patterns in time-series data and commonly employs multilayer perceptrons (MLPs)~\cite{shao2022spatial}, temporal convolutional networks (TCNs)~\cite{yu2017spatio, wu2019graph, ma2025mobimixer}, recurrent neural networks (RNNs)~\cite{bai2020adaptive, wang2020pm2}, or Transformers~\cite{liu2023spatio}. Simultaneously, the spatial module characterizes interactions among nodes, often implemented using graph convolutional networks (GCNs)~\cite{DBLP:journals/pvldb/ShaoZWWXCJ22, ma2025less} or spatial Transformer architectures~\cite{ma2025robust, ma2025causal}. 

Despite the steady progress made by STNNs, most architectures implicitly assume that the spatiotemporal dependencies within the historical input are consistent with those in their future labels, a condition that rarely holds true in complex dynamic systems. Fig.\ref{fig0} shows the resulting spatiotemporal input–label deviation on the LargeST-SD dataset~\cite{liu2024largest}, where predicted labels no longer reflect the trajectories implied by earlier inputs. This deviation typically manifests in two ways:

\textbf{\ding{182}  Spatial deviation} occurs in two complementary situations:
(1) \textit{Similar-inputs but different-labels} (Fig.\ref{fig0} (a)), where two spatial nodes share nearly identical historical observation yet evolve into distinctly different labels; and (2) \textit{Dissimilar-inputs but similar-labels} (Fig.\ref{fig0} (b)), where nodes with markedly different histories eventually converge toward similar future label behaviors. 

\textbf{\ding{183} Temporal deviation} (Fig.\ref{fig0} (c)) arises when a single node undergoes an abrupt behavioral shift that departs from its previous temporal regularities. Each case highlights a distinct limitation of existing STNNs. 

In spatial deviation case (a), the model cannot differentiate future trajectories across nodes that share similar histories, so it fails to capture heterogeneous dynamics; in case (b), it ignores that dissimilar inputs can still converge toward similar outcomes, which encourages redundant learning and weak generalization. Temporal deviation case (c) arises when abrupt shifts break the learned dependencies within a node’s own timeline, causing unstable and inaccurate forecasts.

A potential source of confusion lies in the commonly used terms ``drift” or ``shift” in spatiotemporal out-of-distribution (OOD) scenarios~\cite{wang2024stone, ma2025robust, xia2023deciphering,zhou2023maintaining}, which typically refer to broad distribution discrepancies between the training and test datasets—such as changes in joint statistical properties or spatiotemporal topology. However, this notion is fundamentally different from the form of spatiotemporal input-label deviation in our work. We instead address a more fine-grained deviation between inputs and their future labels—a phenomenon that can manifest not only in OOD scenarios but also within standard independent and identically distributed (IID) settings. To address this unique challenge, existing research has primarily adopted two compensatory strategies. Some studies \cite{gao2023spatial, qi2025lhmformer} have sought to capture richer contextual dynamics by extending the input sequences, a technique aimed at mitigating the temporal aspect of the input-label deviation. Concurrently, other methodologies \cite{shao2022spatial, ma2025less} employ node embeddings to encode the individualized spatial identity of each node, thereby enhancing node differentiability and alleviating the spatial deviation component. Nevertheless, a common limitation among these approaches is their primary emphasis on enhancing input feature extraction, while the direct incorporation of explicit label information remains restricted. This constraint ultimately hinders their effectiveness in fully addressing the input–label deviation problem.

In this paper, we develop a spatiotemporal residual theoretical framework based on the Gaussian-Markov model to support a novel Bidirectional Spatiotemporal Learning Paradigm. The paradigm posits that effective spatiotemporal learning—especially when we explicitly consider label features—must encompass two fundamentally complementary processes: (1) a forward process, which captures the spatiotemporal features of the input data and generates basic predictions, consistent with the conventional STNNs pipeline; and (2) a backward process, which aims to relearn the spatiotemporal deviation information arising from the mismatch between the input and the label, thereby refining and correcting the predictions produced by the forward process.

Based on this theoretical foundation, we designed ReLearner, which can be seamlessly integrated into various STNNs to enhance their backward spatiotemporal learning process. Specifically, ReLearner compares the high-dimensional representations of the input and label to explicit decouple spatiotemporal residual, which is the high-dimensional representation of input-label deviation. These decoupled residuals are then smoothed by a propagation kernel to suppress abnormal signals. Subsequently, ReLearner aligns and decodes the correction terms, which are applied to the basic predictions generated by the forward process to improve accuracy. Throughout this process, we utilize high-dimensional representations of the labels; these information-dense representations enable more comprehensive modeling while simultaneously avoiding direct access to the ground-truth labels. ReLearner effectively models the input-label deviation during training and improves prediction performance during inference. 

The main contributions are summarized as follows:

\begin{itemize}[label=\ding{110}]
    \item \textbf{Novel learning paradigm}. We introduce a spatiotemporal residual theory that establishes a bidirectional learning framework, where the incorporation of label representations into the residual learning process enhances prediction accuracy and model robustness.

    \item \textbf{Universal integration module}. We propose ReLearner, a lightweight yet powerful module that can be readily integrated with diverse STNN architectures to explicitly model spatiotemporal input–label deviation through a reverse learning mechanism.

    \item \textbf{Comprehensive empirical validation}. Extensive experiments on \textbf{11} real-world datasets and \textbf{14} backbone models demonstrate that ReLearner can significantly enhance the prediction performance of existing STNNs, with improvements of up to \textbf{21.18\%}.
\end{itemize}

After this section, we review related work in Section~\ref{gen_inst}. Section~\ref{sec:pre} formally defines the spatiotemporal learning problem. Section~\ref{headings} presents two theoretical formulations grounded in the spatiotemporal Markov field assumption and introduces the proposed ReLearner framework. Section~\ref{sec:exp} provides extensive experiments and visualization analyzes, demonstrating the validity of ReLearner. Section \ref{dis} and Section \ref{con} present future work and conclusions.

\section{Related Work}\label{gen_inst}

\subsection{Spatiotemporal Prediction}

Recently, STNNs are the most representative approaches for spatiotemporal prediction tasks \cite{wang2023pattern,xia2023deciphering,wu2025spatio}, which typically includes a spatial module that captures spatial dependencies and a sequential module that captures temporal dependencies respectively. For example, SSTBAN \cite{guo2023self} follows a multi-task framework by incorporating a self-supervised learner to produce robust latent representations for historical traffic data. STID \cite{shao2022spatial} identified spatiotemporal residual phenomena and proposed utilizing node embeddings to alleviate spatiotemporal residual. However, it can not effectively model label features, thus failing to thoroughly capture spatiotemporal deviation, especially regarding temporal deviation. AGCRN \cite{bai2020adaptive} introduced adaptive graph convolutional recurrent networks that learn node-specific spatial correlations and temporal dynamics through adaptive graph learning and node-wise recurrent units. MTGNN \cite{wu2020connecting} proposed a multi-scale temporal graph neural network that captures both short-term and long-term dependencies in spatiotemporal data using dilated convolutions and graph convolutions.  Recently, researchers have integrated Transformers to take advantage of their powerful long-range modeling capabilities. For example, STAEformer \cite{liu2023spatio} employs Transformers in both spatial and temporal domains. D$^2$STGNN \cite{shao2022decoupled} utilizes Transformer to decouple the diffusion and intrinsic signals in traffic data, thereby improving prediction accuracy.

There are some studies explore non-model approaches such as ST-LoRA \cite{ruan2024low} to improve existing models with node-adaptive low-rank layers, the reported results show limited enhancements.  Additionally, Adaptive Graph Sparsification (AGS) \cite{duan2023localised} and Graph Winning Ticket (GWT) \cite{duan2024pre} algorithms focus on optimizing adjacency matrices in prediction models for improved operational efficiency of Adaptive Spatiotemporal Graph Neural Networks like AGCRN.

\subsection{Out-of-Distribution in Spatiotemporal Learning}
In OOD learning, there are concepts that are easily confused with the spatiotemporal input-label deviation we focus on, such as distribution drift and spatiotemporal shift. We summarize these research advances and highlight their differences. 

Recent studies on distribution shift are primarily from either the spatial or the temporal dimension. From the spatial perspective, several studies have examined spatial drift, where the topology or correlation structure among nodes evolves over time. Representative examples include STONE \cite{wang2024stone} and STOP \cite{ma2025robust}, which dynamically model changing spatial dependencies by masking operations and customized optimization strategies to enhance robustness under varying graph structures. From the temporal perspective, prior works on temporal shift in time series forecasting have focused on non-stationarity in sequential patterns. Typical approaches employ normalization or adaptation mechanisms\cite{ogasawara2010adaptive, passalis2019deep}, such as RevIN~\cite{kim2021reversible}, and Dish-TS~\cite{fan2023dish}, which aim to stabilize temporal statistics or adjust for global distribution drift. While they are effective for capturing spatial topology evolution or temporal discrepancies, respectively, these methods are primarily concerned with dataset-level distribution shifts, and do not explicitly address spatiotemporal input-label deviation that may arise between input features and target labels within the same dataset.

\section{Preliminaries}\label{sec:pre}

\textbf{Spatiotemporal data.} We use a graph $\mathcal{G}=\left(\mathscr{V}, \mathcal{E}, \mathbf{A}\right)$ to represent spatiotemporal data, where $\mathscr{V}$ means the node set with $N$ nodes, $\mathcal{E}$ means the set of edges, and $\mathbf{A} \in \mathbb{R}^{ N \times N}$ is the weighted adjacency matrix of the graph $\mathcal{G}$. We use $x_t \in \mathbb{R}^{N \times f}$ to represent the observed spatiotemporal graph data of $N$ nodes at time step $t$, where $f$ indicates the number of feature channels. 

\textbf{Spatiotemporal prediction.} Given the graph $\mathcal{G}$ and the historical data of the past $T$ time steps $\mathbf{x}=\left\{x_1, ..., x_T\right\}\in\mathbb{R}^{T\times N \times f}$ as inputs, this task aims to learn a function $\mathcal{F}$ that can effectively predict the values (i.e., labels) $\mathbf{y}=\left\{x_{T+1},...,x_{T+T_P}\right\}\in\mathbb{R}^{T_P\times N\times f}$ in further $T_P$ time steps.

\begin{figure*}[t]
    \centering
    \includegraphics[width=0.95\textwidth]{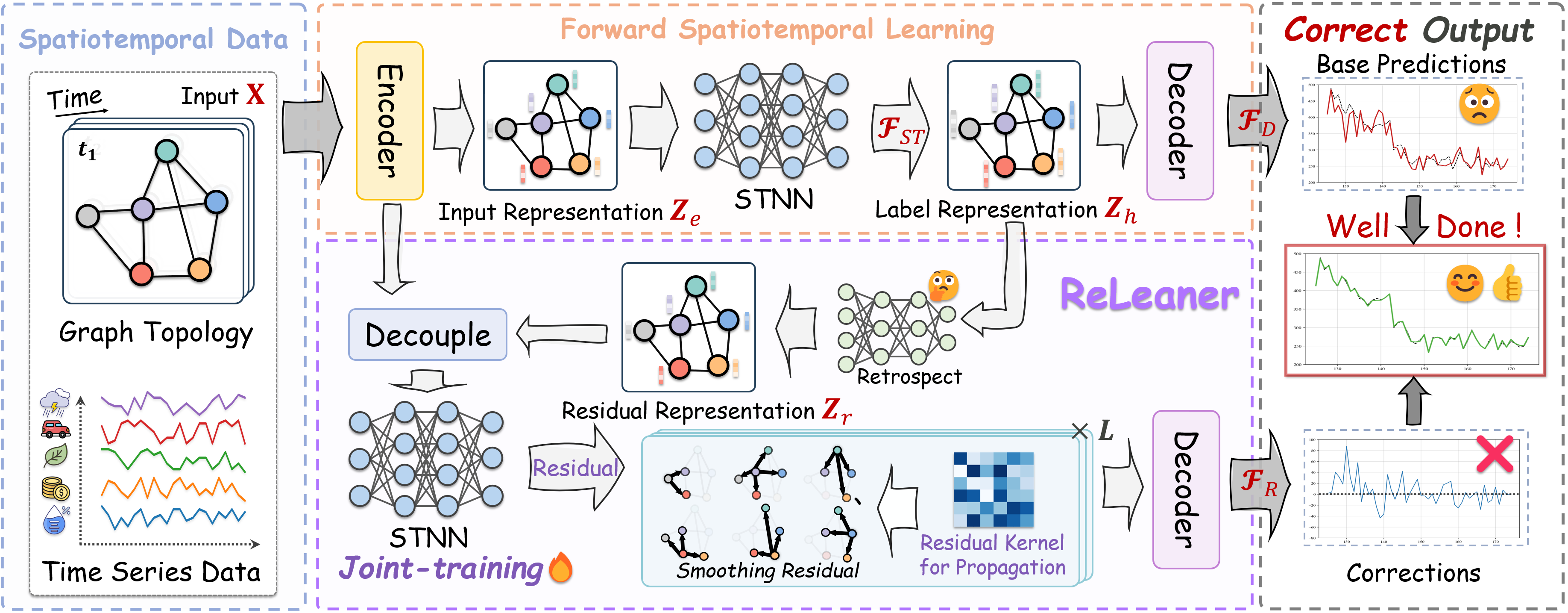} 
    \caption{The overall framework of ReLearner for spatiotemporal learning. ReLearner extends the forward learning process of STNNs by supplementing it with a backward learning process to explicitly model residual information.}
    \label{overmodel} 
\end{figure*}

\section{Method}\label{headings}
We first present the proposed spatiotemporal residual theory, followed by a detailed explanation of the design of ReLearner based on this theory, as illustrated in Fig.\ref{overmodel} and Algorithm~\ref{algorithm1}. Table~\ref{vari_mean} summarizes some important definitions.

\begin{table}[!h]
\setlength{\arrayrulewidth}{1.0pt}
  \renewcommand{\arraystretch}{1.2}
  \centering
  \caption{Some important variables and their definitions.}
  \resizebox{\linewidth}{!}{
    \begin{tabular}{c|l}
    \hline
    \rowcolor[rgb]{.867,.875,  .91}\multicolumn{1}{c|}{\textbf{Variable}} & \multicolumn{1}{c}{\textbf{Definition}} \\
    \hline
    \multicolumn{1}{c|}{$\mathbf{x}$/$\bm{x}$} & \multicolumn{1}{l}{Data/ its  corresponding variable in GMRF} \\
    \rowcolor[rgb]{ .949,  .949,  .949}\multicolumn{1}{c|}{$\mathbf{y}$/$\bm{y}$} & \multicolumn{1}{l}{Label/ its  corresponding variable in GMRF} \\
    \multicolumn{1}{c|}{$W$/$\theta$} & \multicolumn{1}{l}{Parameters of GMRF} \\
    \rowcolor[rgb]{ .949,  .949,  .949}\multicolumn{1}{c|}{$\mathbf{Z}_{e}$} & \multicolumn{1}{l}{Input representation} \\
    \multicolumn{1}{c|}{$\mathbf{Z}_{h}$} & \multicolumn{1}{l}{Label representation} \\
    \rowcolor[rgb]{ .949,  .949,  .949}$\mathbf{Z}_{res}$/$\tilde{\mathbf{Z}}_{res}$ & Residual representation/Smoothed residual representation \\
    $y_{base}$/$y_{corr}$ & Base predition/Prediction correction \\
    \rowcolor[rgb]{ .949,  .949,  .949}$\mathcal{F}_E$ & Input encoder \\
    $\mathcal{F}_{ST}$ & Any STNN \\
    \rowcolor[rgb]{ .949,  .949,  .949}$\mathcal{F}_D$ & Base encoder \\
    $\mathcal{F}_{R}$ & Residual decoder \\
    \rowcolor[rgb]{ .949,  .949,  .949}$T$   & The length of input time step \\
    $N$   & The number of nodes \\
    \rowcolor[rgb]{ .949,  .949,  .949}$T_p$ & The length of label time step \\
    $f$   & The number of features of spatiotemporal data \\
    \rowcolor[rgb]{ .949,  .949,  .949}$K$   & The number of kernels in residual propagation kernel \\
    $L$   & The number of residual propagation layers \\
    \hline
    \end{tabular}}%
  \label{vari_mean}%
\end{table}%

\subsection{Spatiotemporal Residual Theory based on Gaussian Markov Random Field}
Gaussian Markov Random Field (GMRF) is a powerful framework capable of modeling complex dependencies among random variables in a structured and interpretable manner, and it has been widely applied in spatiotemporal dynamics analysis \cite{zheng2016traffic,furtlehner2021short}. Inspired by these pioneering studies, we adopt GMRF to analyze the latent spatiotemporal dependencies within the data. In this framework, each spatiotemporal data point is associated with a corresponding random variable, and a transition matrix is used to represent the dependencies among these variables. \textbf{In the following sections, we use roman (upright) font to denote spatiotemporal data points, and \textit{italic font} to denote the corresponding random variables in the GMRF.} For example, $\mathbf{x}$ denotes a spatiotemporal data point, while $\bm{x}$ represents its associated random variable in the GMRF. Before formally introducing the Gaussian Markov Random Field, we first provide the definition of its underlying Gaussian distribution.

\noindent \textbf{Definition 1. Multivariate Gaussian Distribution.} Consider the multivariate Gaussian distribution \cite{goodman1963statistical} on spatiotemporal variables $\bm{T}\sim\mathcal{N}(\bar{\mathbf{T}},\mathbf{\Sigma})$, including historical and future temporal variables, where $
\bar{\mathbf{T}}$ is the expectation and $\mathbf{\Sigma}$ is the covariance matrix. The probability density of $\bm{T}$ is
\begin{align}
f_{\bm{T}}\left(\mathbf{T}|\bar{\mathbf{T}},\mathbf{\Sigma}\right)=\frac{\det\left(\mathbf{\Sigma}\right)^{\frac12}\exp\left({-\frac12\tilde{\mathcal{V}}(\mathbf{T})^\top\mathbf{\Sigma}^{-1}\tilde{\mathcal{V}}(\mathbf{T})}\right)}{\left(\sqrt{2\pi}\right)^{{\left(T+T_P\right)N}}},
\end{align}
where $\tilde{\mathcal{V}}(\mathbf{T})=\mathcal{V}\left(\mathbf{T}\right)-\mathcal{V}(\bar{\mathbf{T}})$. We use the bold subscript symbols $\bm{x}$ and $\bm{y}$ to denote the subtensor of the tensor corresponding to the rows and columns of the $\bm{x}$ and $\bm{y}$ parts of $\bm{T}$. Split the variable into historical and future temporal parts; then the spatiotemporal variables have the distribution in block form.
\begin{equation}
\left[\bm{x},\bm{y}\right]\sim\mathcal{N}\left(\left[\mathbf{\bar{x}},\mathbf{\bar{y}}\right],\begin{bmatrix}\mathbf{\Sigma}_{\bm{x}\bm{x}}&\mathbf{\Sigma}_{\bm{x}\bm{y}}\\\mathbf{\Sigma}_{\bm{y}\bm{x}}&\mathbf{\Sigma}_{\bm{y}\bm{y}}\end{bmatrix}\right).
\end{equation}
In many cases, the covariance matrix is dense while the precision matrix (or inverse covariance matrix) $\mathbf{\Gamma}=\mathbf{\Sigma}^{-1}$ is sparse \cite{fan2016overview}, hence it is oftentimes economical to work with $\mathbf{\Gamma}$. We rewrite the block form distribution of spatiotemporal variables with precision matrix.
\begin{equation}
\left[\bm{x},\bm{y}\right]\sim\mathcal{N}\left(\left[\mathbf{\bar{x}},\mathbf{\bar{y}}\right],\begin{bmatrix}\mathbf{\Gamma}_{\bm{x}\bm{x}}&\mathbf{\Gamma}_{\bm{x}\bm{y}}\\\mathbf{\Gamma}_{\bm{y}\bm{x}}&\mathbf{\Gamma}_{\bm{y}\bm{y}}\end{bmatrix}^{-1}\right).
\end{equation}

After establishing the definition of the multivariate Gaussian distribution, we proceed to derive its marginal and conditional forms, which serve as key mathematical foundations for the subsequent Forward spatiotemporal learning (Proposition~\ref{theory1}) and spatiotemporal residual theorems (Proposition~\ref{theory2}). For clarity and completeness, we first present the formal definitions of these two distributions below.

\noindent\ding{115} \textbf{Marginal distribution.} The marginal distribution of future temporal variable $\bm{y}$ is simply self-relevant from the mean and covariance
\begin{equation}
\bm{y}\sim\mathcal{N}\left(\mathbf{\bar{y}},\mathbf{\Sigma}_{\bm{y}\bm{y}}\right).
\end{equation}
By the property of the inverse of block matrix \cite{choi2009new}, we have $\mathbf{\Sigma}_{\bm{y}\bm{y}}=\left(\mathbf{\Gamma}_{\bm{y}\bm{y}}-\mathbf{\Gamma}_{\bm{y}\bm{x}}\mathbf{\Gamma}_{\bm{x}\bm{x}}^{-1}\mathbf{\Gamma}_{\bm{x}\bm{y}}\right)^{-1}$, hence we rewrite the marginal distribution of future temporal variable in the form of precision matrix,
\begin{equation}
\bm{y}\sim\mathcal{N}\left(\mathbf{\bar{y}},\left(\mathbf{\Gamma}_{\bm{y}\bm{y}}-\mathbf{\Gamma}_{\bm{y}\bm{x}}\mathbf{\Gamma}_{\bm{x}\bm{x}}^{-1}\mathbf{\Gamma}_{\bm{x}\bm{y}}\right)^{-1}\right).
\end{equation}

\noindent\ding{115} \textbf{Conditional distribution.}\label{conditionaldistribution}~The conditional distribution of future temporal variable $\bm{y}$ with respect to history temporal variable $\bm{x}=\mathbf{x}$ is also a multivariate Gaussian distribution
\begin{equation}
\begin{aligned}
\bm{y}|\bm{x}=\mathbf{x}\sim\mathcal{N}\left(\mathbf{\bar{y}}+\mathbf{\Sigma}_{\bm{y}\bm{x}}\mathbf{\Sigma}_{\bm{x}\bm{x}}^{-1}\left(\mathbf{x}-\bar{\mathbf{x}}\right),\mathbf{\Sigma}_{\bm{y}\bm{y}}-\mathbf{\Sigma}_{\bm{y}\bm{x}}\mathbf{\Sigma}_{\bm{x}\bm{x}}^{-1}\mathbf{\Sigma}_{\bm{x}\bm{y}}\right).
\end{aligned}
\end{equation}
Moreover, one can show that $\mathbf{\Sigma}_{\bm{yx}}\mathbf{\Sigma}_{\bm{xx}}^{-1}=-\mathbf{\Gamma}_{\bm{yy}}^{-1}\mathbf{\Gamma}_{\bm{yx}}$ and $\left(\mathbf{\Sigma}_{\bm{y}\bm{y}}-\mathbf{\Sigma}_{\bm{y}\bm{x}}\mathbf{\Sigma}_{\bm{x}\bm{x}}^{-1}\mathbf{\Sigma}_{\bm{x}\bm{y}}\right)=\mathbf{\Gamma}_{\bm{y}\bm{y}}^{-1}$ by the block matrix inversion, then the conditional distribution can be written as
\begin{equation}
\bm{y}|\bm{x}=\mathbf{x}\sim\mathcal{N}\left(\mathbf{\bar{y}}-\mathbf{\Gamma}_{\bm{yy}}^{-1}\mathbf{\Gamma}_{\bm{yx}}\left(\mathbf{x}-\bar{\mathbf{x}}\right),\mathbf{\Gamma}_{\bm{y}\bm{y}}^{-1}\right).
\end{equation}

Subsequently, we can derive the GMRF properties of the spatiotemporal input $\mathbf{x}$ and label $\mathbf{y}$ by the multivariate Gaussian distribution assumption. First, we stack 
the input data $\mathbf{x}$ and the label $\mathbf{y}$ along the temporal dimension into a tensor $\mathbf{T}\coloneqq\left[\mathbf{x},\mathbf{y}\right]\in\mathbb{R}^{\left(T+T_P\right)\times N\times f}$. We use $\mathbf{T}_{t,:,:}\in\mathbb{R}^{N\times f}$ to denote spatiotemporal data of all nodes at $t$-th time step. We also use $\mathbf{T}_{:,u,:}\in\mathbb{R}^{\left(T+T_P\right)\times f}$ to represent the spatiotemporal data of $u$-th node during all time steps. As mentioned above, the random variable of $\mathbf{T}$ in our GMRF is denoted as $\bm{T}$. In a GMRF, all values in the matrix $\mathbf{T}$ are jointly sampled from a distribution over the random variable $\bm{T}$. The definition of the joint distribution of $\bm{T}$ in the GMRF is given below.
\begin{tcolorbox}
[colback=Salmon!20, colframe=Salmon!90!Black, title=\begin{definition}\label{def1}The Joint Distribution of $\bm{T}$ in the GMRF\end{definition}]

The joint distribution of $\bm{T}$ in the GMRF is decided by a probability density function\cite{baz2022some,rue2005gaussian}.
\begin{align}
    f_{\bm{T}}\left(\bm{T}=\mathbf{T}\mid W, \theta\right)=\frac{e^{-\Phi\left(\mathbf{T}\mid W, \theta\right)}}{\int d\mathbf{T}^{\prime}e^{-\Phi\left(\mathbf{T}^{\prime}\mid W, \theta\right)}},
\end{align}
where $W\in\mathbb{R}^{\left(T+T_P\right)\times\left(T+T_P\right)}$ and $\theta\in\mathbb{R}^{\left(T+T_P\right)}$ are the parameters of GMRF. 
\end{tcolorbox}

\noindent Here $W$ should be symmetric positive definite, describing the conditional dependency structure between variables, and $\theta$ is entry-wise positive, which defines the expected value of each variable. The exponent power function $\Phi$ is defined as
\begin{equation}
\begin{aligned}
&\Phi\left(\mathbf{T}\mid W,\theta\right)\\
\coloneqq&\frac{1}{2}\sum_{u\in V}\mathbf{T}_{:,u,:}^\top W\mathbf{T}_{:,u,:}+\frac{1}{2}\sum_{t=1}^{T+T_P}\theta_t\mathbf{T}_{t,:,:}^\top\mathcal{A}\left(\mathbf{A}\right)\mathbf{T}_{t,:,:}\\
=&\frac{1}{2}\mathcal{V}\left(\mathbf{T}\right)^\top\mathbf{\Gamma}\mathcal{V}\left(\mathbf{T}\right),
\end{aligned}
\end{equation}
where the potential matrix $\mathbf{\Gamma}\in\mathbb{R}^{\left[\left(T+T_P\right)N\right]\times\left[\left(T+T_P\right)N\right]}$ reflects the dependence of variables of GRMF in temporal and spatial dimensions, which can be computed as
\begin{equation}
\mathbf{\Gamma}\coloneqq\left(W\otimes\mathbf{I}_N\right)+\operatorname{diag}\left(\theta\right)\otimes\mathcal{A}\left(\mathbf{A}\right).
\end{equation}
Here $\mathcal{A}\left(\mathbf{A}\right)=\mathbf{I}_N-\mathcal{N}\left(\mathbf{A}\right)$ is a graph Laplace-like operator, and $\mathbf{I_N}$ is the identity matrix of the adjacency matrix $\mathbf{A}$, $\mathcal{N}\left(\cdot\right)$ is a normalization operator such as the normalized graph Laplacian $\mathcal{A}\left(\mathbf{A}\right)=\mathbf{I_N}-\mathbf{D}^{-1/2}\mathbf{A}\mathbf{D}^{-1/2}$ with degree matrix $\mathbf{D}=\operatorname{diag}\left(\sum_{u\in V}\mathbf{A}_{1,u},...,\sum_{u\in V}\mathbf{A}_{N,u}\right)$ and diagonalization operator $\operatorname{diag}\left(\cdot\right)$. $\mathcal{V}\left(\cdot\right)$ is a vectorization operator to unfold the first two dimensions of the input, i.e., $\mathcal{V}\left(\mathbf{T}\right)=\left(\mathbf{T}_{1,1,:},...,\mathbf{T}_{T+T_P,N,:}\right)^\top\in\mathbb{R}^{\left[\left(T+T_P\right)N\right]\times f}$, and $\otimes$ is the Kronecker product.

Based on the constructed GMRF, the prediction process of existing spatiotemporal learning models can be interpreted as learning the conditional distribution of future variables $\mathbf{y}$ with respect to historical input data $\mathbf{x}$, i.e., modeling $\mathbf{y},|,\mathbf{x}=\mathbf{x}$. Within the GMRF framework, this conditional distribution admits a closed-form representation that naturally decomposes into a series of spatiotemporal operations. The following proposition provides the corresponding formulation in detail.

\begin{tcolorbox}[colback=Emerald!10,colframe=cyan!40!black,title=\begin{proposition}\label{theory1}Forward spatiotemporal learning\end{proposition}]
For any further time step $t= \{1, 2, \cdots , T_P\}$, the expectation of the variable $\bm{y}_t \in\mathbb{R}^{1\times N\times f}$ representing labels all nodes at $t$-th future time step with respect to the input data $\mathbf{x}$ can be computed as:
\begin{equation}
\begin{aligned}\label{formula2}
\hspace{-0.5em}\mathbb{E}\left[\bm{y}_t|\mathbf{x}\right]&=\left(1-\gamma_t\right)\sum_{k=0}^{\infty}\left(\gamma_t\mathcal{N}\left(\mathbf{A}\right)\right)^k\mathbf{x}^\top\times_2\bm{\beta}_t^\top,\\
&=\left(1-\gamma_t\right)\sum_{k=0}^{\infty}\left(\gamma_t\mathcal{N}\left(\mathbf{A}\right)\right)^k\times_2\left(\bm{\beta}_t\mathbf{x}\right)^\top,
\end{aligned}
\end{equation}
where $\gamma_t=\frac{\theta_t^\prime}{W_{t^\prime,t^\prime}+\theta_t^\prime}\in\mathbb{R}$ is a scaling scalar and $\bm{\beta}_{t}=-\frac{W_{t^\prime,1:T}}{W_{t^\prime,t^\prime}}\in\mathbb{R}^{1\times T}$ is a coefficient vector.
\end{tcolorbox}
\noindent Here $W_{t^\prime,t^\prime}$ means the value of $t^\prime$-th row and $t^\prime$-th column of $W$ with $t^\prime=T+t$ and $W_{t^\prime,1:T}$ means the first $T$ column and $t^\prime$-th row of $W$. $\times_2$ implies performing tensor multiplication operations in the second dimension.

\noindent\textit{Proof of Proposition \ref{theory1}}.  By the definition~\ref{def1} of GMRF, we can define the multivariate Gaussian distribution of probability density function,
\begin{align}
f_{\bm{T}} \left(\mathbf{T}\right)=\frac{\det\left(\mathbf{\Gamma}^{-1}\right)^{\frac12}\exp\left({-\frac12\mathcal{V}\left(\mathbf{T}\right)^\top\mathbf{\Gamma}\mathcal{V}\left(\mathbf{T}\right)}\right)}{\left(2\pi\right)^{\frac{N\left(T+T_P\right)}{2}}},
\end{align}
where
\begin{equation}
\mathbf{\Gamma}=\begin{bmatrix}\mathbf{\Gamma}_{\bm{x}\bm{x}}&\mathbf{\Gamma}_{\bm{x}\bm{y}}\\\mathbf{\Gamma}_{\bm{y}\bm{x}}&\mathbf{\Gamma}_{\bm{y}\bm{y}}\end{bmatrix}=\mathbf{\Sigma}^{-1},
\end{equation}
is the precision matrix, i.e., the inverse of covariance matrix $\mathbf{\Sigma}$. The temporal tensor are jointly sampled via multivariate Gaussian distribution $\mathcal{V}\left(\bm{T}\right)\sim\mathcal{N}(\mathbf{0},\mathbf{\Gamma}^{-1})$. Here, $W$ satisfying symmetric positive definite and $\theta$ satisfying entry-wise positive are the pseudo parameters of standard GMRF model. Hence we have $\bm{y}|\mathbf{x}\sim\mathcal{N}\left(-\mathbf{\Gamma}_{\bm{yy}}^{-1}\mathbf{\Gamma}_{\bm{yx}}\mathcal{V}\left(\mathbf{x}\right),\mathbf{\Gamma}_{\bm{y}\bm{y}}^{-1}\right)$, i.e., 
\begin{align}
    &\mathbb{E}\left[\bm{y}|\mathbf{x}\right]=-\mathbf{\Gamma}_{\bm{yy}}^{-1}\mathbf{\Gamma}_{\bm{yx}}\mathcal{V}\left(\mathbf{x}\right)\\\notag
    &=-\left(W_{\bm{yy}}\otimes\mathbf{I}_N+\operatorname{diag}\left(\theta_{\bm{y}}\right)\otimes\mathcal{A}\left(\mathbf{A}\right)\right)^{-1}\left(W_{\bm{yx}}\otimes\mathbf{I}_N\right))\mathcal{V}\left(\mathbf{x}\right).
\end{align}
Hence for arbitrary $t\in{1,2,...,T_P}$, we have 
\begin{align}
&\mathbb{E}\left[\bm{y}_t|\mathbf{x}\right]\\&=-\left(W_{T+t,T+t}\mathbf{I}_N+\theta_{T+t}\mathcal{A}\left(\mathbf{A}\right)\right)^{-1}\left(W_{T+t,1:T}\otimes\mathbf{I}_N\right)\mathcal{V}\left(\mathbf{x}\right),\notag\\&=-\left(W_{T+t,T+t}\mathbf{I}_N+\theta_{T+t}\mathcal{A}\left(\mathbf{A}\right)\right)^{-1}\mathbf{x}^\top\times_2W_{T+t,1:T}^\top,\notag\\
&=\left(W_{T+t,T+t}\mathbf{I}_N+\theta_{T+t}\mathcal{A}\left(\mathbf{A}\right)\right)^{-1}\times_2\left(-W_{T+t,1:T}\mathbf{x}\right)^\top.\notag
\end{align}
where $\times_i$ is the matrix multiplication of tensor on the $i$-th dimension. Let $\alpha_t=\frac{\theta{t}}{W_{T+t,T+t}}$ and $\bm{\beta}_t=-\frac{W_{T+t,1:T}}{W_{T+t,T+t}}$, we obtain reduce the above equation like
\begin{equation}
\begin{aligned}
\mathbb{E}\left[\bm{y}_t|\mathbf{x}\right]&=\left(\mathbf{I}_N+\alpha_t\mathcal{A}\left(\mathbf{A}\right)\right)^{-1}\mathbf{x}^\top\times_2\bm{\beta}_t^\top,\\
&=\left(\mathbf{I}_N+\alpha_t\mathcal{A}\left(\mathbf{A}\right)\right)^{-1}\times_2\left(\bm{\beta}_t\mathbf{x}\right)^\top.
\end{aligned}
\end{equation}
Moreover, since $\operatorname{lim}_{k\rightarrow\infty}{\mathcal{N}\left(A\right)}^{k}=\mathbf{0}$, we expand the term $\left(\mathbf{I}_N+\alpha_t\mathcal{A}\left(\mathbf{A}\right)\right)^{-1}$ in terms of the Neumann series \cite{moulinec2018convergence} as
\begin{equation}
\begin{aligned}
\left(\mathbf{I}_N+\alpha_t\mathcal{A}\left(\mathbf{A}\right)\right)^{-1}&=\left(\left(1+\alpha_t\right)\mathbf{I}_N-\alpha_t\mathcal{N}\left(\mathbf{A}\right)\right)^{-1},\\
&=\left(1-\gamma_t\right)\sum_{k=0}^{\infty}\left(\gamma_t\mathcal{N}\left(\mathbf{A}\right)\right)^k.
\end{aligned}
\end{equation}
where $\gamma_t=\alpha_t/\left(1+\alpha_t\right)$. Hence for $\forall t\in{1,2,...,T_P}$, we get 
\begin{equation}
\begin{aligned}
\mathbb{E}\left[\bm{y}_t|\mathbf{x}\right]&=\left(1-\gamma_t\right)\sum_{k=0}^{\infty}\left(\gamma_t\mathcal{N}\left(\mathbf{A}\right)\right)^k\mathbf{x}^\top\times_2\bm{\beta}_t^\top,\\
&=\left(1-\gamma_t\right)\sum_{k=0}^{\infty}\left(\gamma_t\mathcal{N}\left(\mathbf{A}\right)\right)^k\times_2\left(\bm{\beta}_t\mathbf{x}\right)^\top,
\end{aligned}
\end{equation}
which completes the proof. $\hfill\qedsymbol$

Equation~\ref{formula2} outlines a general learning paradigm employed by existing STNNs. In this paradigm, the STNN extracts features from both spatial and temporal dimensions of the input data $\mathbf{x}$ to produce predictions. The spatial operator $\sum_{k=0}^{\infty}\left(\gamma_t\mathcal{N}\left(\mathbf{A}\right)\right)^k$ corresponds to the graph convolution operation, utilizing GCN \cite{shao2022pre} or Transformer \cite{jiang2023pdformer} as operators. Conversely, the temporal operator $\bm{\beta}_t$ captures the correlation across multiple time steps. This component typically functions as a sequence modeling model, such as TCN \cite{bai2018empirical}, RNN \cite{cheng2024spatial}, or Transformer \cite{wang2013dynamic}.

The aforementioned traditional paradigms focus only on the spatiotemporal features of the input data. To tackle this issue, we are interested in exploring the integration of label features into the learning process. Let $\mathbf{y}_{t, u}\in \mathbb{R}^{f}$ represent the label of node $u$ at time step $t$, and denote the labels of the other nodes (excluding node $u$) as $\mathbf{y}_{t, \hat{u}}\coloneqq{\left[\mathbf{y}_{t, 1}^\top,...,\mathbf{y}_{t, u-1}^\top,\mathbf{y}_{t, u+1}^\top...,\mathbf{y}_{N}^\top\right]}^\top\in\mathbb{R}^{\left(N-1\right)\times f}$. There exist correlations between $\bm{y}_{t, u}$ and the labels of the other nodes, considering only the spatial correlation at each time step and disregarding the dependence across different time steps. Our objective is to incorporate the spatiotemporal features into the GMRF framework. The condition for the GMRF is to predict the variable $\bm{y}_{t,u}$ with the goal of minimizing the difference from the label $\mathbf{y}_{t, u}$. 
\begin{tcolorbox}[colback=Emerald!10,colframe=cyan!40!black,title=\begin{proposition}\label{theory2}Spatiotemporal Residual Theory\end{proposition}]
For any future time step $t=\{1,2,\cdots,T_P\}$, the expectation of $\bm{y}_{t,u}$ with respect to $\mathbf{x}$ and $\mathbf{y}_{t,\hat{u}}$ is
\begin{align}\label{eq_43}
\mathbb{E}&\left[\bm{y}_{t,u}|\mathbf{x},\mathbf{y}_{t,\hat{u}}\right]
\\
=&\underbrace{\mathbb{E}\left[\bm{y}_{t, u}|\mathbf{x}\right]}_{\text{Base prediction}}+\underbrace{\underbrace{\beta_{t,u}\left(\mathbf{I}_N+\alpha_t\mathcal{A}\left(\mathbf{A}\right)\right)_{u,\hat{u}}}_{\text{Propagation Kernel}}\times_2\underbrace{\mathbf{r}_{t,\hat{u}}}_{\text{Residual}}}_{\text{Correction}}.\notag
\end{align}
\end{tcolorbox}
\noindent\textit{Proof of Proposition \ref{theory2}}. Without loss of generality, we simplify the subsequent calculations by assuming two nodes disjoint union partition $V=V_1\cup V_2$, i.e., $V_1\cap V_2=\emptyset$ to explore what are the implications for spatiotemporal learning when adding future impacts between data. Recall the result of Theory 1 and above proof, we get 
$\bm{y}|\mathbf{x}\sim\mathcal{N}\left(\mathbb{E}\left[\bm{y}|\mathbf{x}\right],\mathbf{\Gamma}_{\bm{y}\bm{y}}^{-1}\right).$
Hence the conditional distribution of $\bm{y}_{t,V_1}$ respect to $\mathbf{y}_{t,V_2}$ and $\mathbf{x}$ for the disjoint union $V_1\cup V_2$ of node set $V$ and arbitrary $t=1,2,...,T_P$ is $\bm{y}_{t,V_1}|\mathbf{x},\mathbf{y}_{t,V_2}\sim\mathcal{N}\left(\mu^*,\mathbf{\Gamma}_{t,V_1V_1}^{-1}\right)$, The expectation term $\mu^*$ satisfies,
\begin{align*}
\mu^*=\mathbb{E}\left[\bm{y}_{t,V_1}|\mathbf{x}\right]&+\mathbf{\Gamma}_{t,V_1V_1}^{-1}\mathbf{\Gamma}_{t,V_1V_2}\times_2\left(\mathbb{E}\left[\bm{y}_{t,V_1}|\mathbf{x}\right]-\mathbf{y}_{t,V_2}\right),
\end{align*}
where $\mathbf{y}_{t,V_i}\coloneqq{\left[\mathbf{y}_{t,u,:}^\top\mid\forall u\in V_i\right]}^\top$ for $i=1,2$. Hence the above expectation is,
\begin{equation}
\begin{aligned}
&\mathbb{E}\left[\bm{y}_{t,V_1}|\mathbf{x},\mathbf{y}_{t,V_2}\right]\\
=~&\mathbb{E}\left[\bm{y}_{t,V_1}|\mathbf{x}\right]+\mathbf{\Gamma}_{t,V_1V_1}^{-1}\mathbf{\Gamma}_{t,V_1V_2}\left(\mathbb{E}\left[\bm{y}_{t,V_1}|\mathbf{x}\right]-\mathbf{y}_{t,V_2}\right),
\end{aligned}
\end{equation}
and the right term is calculated as follows,
\begin{equation}
    \begin{aligned}
&\mathbf{\Gamma}_{t,V_1V_1}^{-1}\mathbf{\Gamma}_{t,V_1V_2}\left(\mathbb{E}\left[\bm{y}_{t,V_1}|\mathbf{x}\right]-\mathbf{y}_{t,V_2}\right),\\
=~&\left(W_{T+t,T+t}\mathbf{I}_N+\theta_{T+t}\mathcal{A}\left(\mathbf{A}\right)\right)_{V_1V_1}^{-1}\\
&\left(W_{T+t,T+t}\mathbf{I}_N+\theta_{T+t}\mathcal{A}\left(\mathbf{A}\right)\right)_{V_1V_2}\times_2\mathbf{r}_{t,V_2},\\
=~&\left(\mathbf{I}_N+\alpha_{t}\mathcal{A}\left(\mathbf{A}\right)\right)_{V_1V_1}^{-1}\left(\mathbf{I}_N+\alpha_{t}\mathcal{A}\left(\mathbf{A}\right)\right)_{V_1V_2}\times_2\mathbf{r}_{t,V_2},\\
=~&\left(1-\gamma_t\right)\sum_{k=0}^{\infty}\left(\gamma_t\mathcal{N}\left(\mathbf{A}\right)_{V_1,V_1}\right)^k\\
&\left(\mathbf{I}_N+\alpha_t\mathcal{A}\left(\mathbf{A}\right)\right)_{V_1,V_2}\times_2\mathbf{r}_{t,V_2},  
    \end{aligned}
\end{equation}
This term is still from the expansion of Neumann series \cite{moulinec2018convergence} where $\alpha_t=\frac{\theta_{t}}{W_{T+t,T+t}}$ and $\gamma_t=\alpha_t/\left(1+\alpha_t\right)$. The term $\left(\mathbf{I}_N+\alpha_t\mathcal{A}\left(\mathbf{A}\right)\right)_{V_1,V_2}$ indicates the submatrix consisting of rows corresponding to entries in $V_1$ and columns corresponding to entries in $V_2$ for $\mathbf{I}_N+\alpha_t\mathcal{A}\left(\mathbf{A}\right)$, similar to $\mathcal{N}\left(\mathbf{A}\right)_{V_1,V_1}$, which illustrates the dynamics of residual propagation in this context. It must be noted, however, that the results of the closed form are independent of the node disjoint union partition chosen, as determined by the equivariance of the GMRF \cite{baz2022some}. Hence, the case we considered in Spatiotemporal Residual Theory is just a special example in the proof when $V_1=\left\{u\right\}$ and $V_2=V\backslash\left\{u\right\}$.$\hfill\qedsymbol$

The base prediction is generated by forward spatiotemporal learning corresponding to Proposition \ref{theory1}. The smoothing coefficient is used to smooth the residual term for smooth learning. And  $\tau_{t,u}=\left[\left(1+\alpha_t\right)\left(1+\alpha_t\mathcal{A}\left(\mathbf{A}\right)_{u,u}\right)\right]^{-1}$ is a scalar with $\alpha_t=\frac{\theta_{t^\prime}}{W_{t^\prime,t^\prime}}$ and $t^\prime=T+t$, where $\mathcal{A}\left(\mathbf{A}\right)_{u,u}$ indicates the entry on $u$-th row and $u$-th column of $\mathcal{A}\left(\mathbf{A}\right)$, and $\left(\mathbf{I}_N+\alpha_t\mathcal{A}\left(\mathbf{A}\right)\right)_{u,\hat{u}} \in\mathbb{R}^{1\times \left(N-1\right)}$ is the $u$-th row of $\mathbf{I}_N+\alpha_t\mathcal{A}\left(\mathbf{A}\right)$ excluding itself. The smoothing coefficient signifies the affinity between node $u$ and the remaining nodes. In fact, it can be regarded as a part from the graph kernel. To differentiate, we call this graph kernel as residual propagation kernel. The residual term $\mathbf{r}_{t,\hat{u}}$ represents the difference between predicted expectations and labels, which is denoted as follows:
\begin{align}\label{res_term}
\mathbf{r}_{t,\hat{u}}\coloneqq\mathbb{E}\left[\bm{y}_{t, \hat{u}}|\mathbf{x}\right]-\mathbb{E}\left[\mathbf{y}_{t,\hat{u}}\right] \in\mathbb{R}^{1\times \left(N-1\right)\times f}.
\end{align}

In Equation~\ref{res_term}, the base prediction $\mathbb{E}\left[\bm{y}_{t, \hat{u}}|\mathbf{x}\right]$ arises from modeling the correlations within the input data, while $\mathbb{E}\left[\mathbf{y}_{t,\hat{u}}\right]$ depends on the autocorrelation of the labels.  Therefore, the residual term actually represents the feature difference between the input and the label features.

\noindent\textbf{Remark}. Spatiotemporal Residual Theory unveils that a holistic spatiotemporal learning paradigm integrating label information should encompass both a forward process and a backward process.   The forward process in spatiotemporal learning captures the inter-dependencies in the input data to produce base predictions, whereas the backward process is dedicated to modeling residuals for generating correction terms.   These residuals encapsulate the discrepant features of the input and labels.   The ultimate prediction is derived from the amalgamation of the base prediction and the correction terms. This theory aligns perfectly with our assertion.

\subsection{ReLearner based on Spatiotemporal Residual Theory}
We initially outline the general structure of conventional STNNs utilized for forward spatiotemporal learning.  Subsequently, we delve into a comprehensive description of our innovative ReLearner module and elucidate its seamless integration with STNNs.

\subsubsection{Forward Spatiotemporal Learning}

As shown in Fig.\ref{overmodel}, existing spatiotemporal prediction models typically consist of three parts: (1) An input encoder maps the input data into a high-dimensional feature space, generally combining enhancement strategies such as node embedding technology. This function is denoted as $\mathcal{F}_E: \mathbf{x}\mapsto\mathbf{Z}_{e}\in\mathbb{R}^{T\times N \times d_{e}}$, where $\mathbf{Z}_{e}$ is termed as the \textit{input representation}; (2) A STNN module $\mathcal{F}_{ST}$ is used to capture spatiotemporal features of input representation and generate the \textit{label representation} $\mathbf{Z}_{h}$: $\mathcal{F}_{ST}:\mathbf{Z}_{e}\mapsto\mathbf{Z}_{h}\in\mathbb{R}^{T\times N \times d_{h}}$. The trained label representations $\mathbf{Z}_{h}$ can approximate the high-dimensional feature mapping of the labels; (3) A base decoder $\mathcal{F}_D$ decodes the label representation $\mathbf{Z}_{h}$ to generate base prediction $\mathbf{y}_{base}$, $\mathcal{F}_D:\mathbf{Z}_{h}\mapsto\mathbf{y}_{base}\in\mathbb{R}^{T_P\times N \times f}$. 

\subsubsection{Backward Residual Correction} The proposed backward spatiotemporal deviation learning pipeline is composed of three key components: \textbf{residual learning}, which extracts discrepancies between input feature representation and predictive labels representation; \textbf{residual propagation}, which captures and smooths spatiotemporal dependencies among these residuals; and \textbf{prediction correction}, which leverages the propagated residual information to adaptively correct the model’s forecasts and enhance overall predictive accuracy.

\noindent\ding{182} \textbf{Residual Learning.} In Equation \ref{res_term}, the residual term delineates the feature disparity between the input and the labels. Our methodology involves modeling the input representation $\mathbf{Z}_{e}$ and the label representation $\mathbf{Z}_{h}$ to calculate this residual. The initial label representations are derived through spatiotemporal learning of input data.  Post-training, these representations closely mirror the distribution of the labels, akin to high-dimensional feature mappings of the labels.  Leveraging these representations for residual computation allows us to capture deviation across diverse dimensions.  Furthermore, this strategy alleviates the requirement to access the real labels, particularly in situations where acquiring the true labels is impractical during the inference phase.

Specifically, we employ a MLP layer with Gaussian Error Linear Units (GELU) activation function \cite{hendrycks2016gaussian,devlin2018bert} to map the label representation $\mathbf{Z}_{h}$ to the same space: $\mathbf{Z}_{h}\mapsto \mathbf{Z}_{r}\in\mathbb{R}^{T\times N \times d_{e}}$. Subsequently, we decouple the spatiotemporal deviation features by subtracting the two representations $\mathbf{Z}_{e}-\mathbf{Z}_{r}$, which helps filter out redundant spatiotemporal features. This deviation is then fed into the STNN module $\mathcal{F}_{ST}$, placing particular emphasis on the model's relearning of this information. Consequently, the resulting output $\mathbf{Z}_{res}$ represents the residual term. The overall calculation process can be outlined as follows:
\begin{align}
\mathbf{Z}_{res}&=\mathcal{F}_{ST}\left(\mathbf{Z}_{e}-\text{MLP}\left(\mathbf{Z}_{h}\right)\right).
\end{align}
\noindent\ding{183} \textbf{Residual Propagation.} As show in Spatiotemporal Residual Theory, it is essential to use spatiotemporal correlation to smooth this residual term $\mathbf{Z}_{res}$. The smoothing process is similar to an aggregation process based on a residual propagation kernel, i.e., graph kernel. In this paper, we thoroughly investigate the effectiveness of different types of graph kernels, including predefined kernel, diffusion kernel, adaptive kernel, and data-driven kernel (details in Section~\ref{ap_kernel}). In order to enhance the representation ability of models, we deploy $K$ kernels $\left(\mathcal{K}_1,\mathcal{K}_2,...,\mathcal{K}_K\right)$. Let $\mathcal{K}\coloneqq\bm{\tau}\left(\mathbf{I}_N+\frac1K\sum_{i=1}^K\bm{\alpha_i}\mathcal{K}_i\right)$, where $\bm{\alpha}_i=\operatorname{diag}\left(\alpha_{i,1},\alpha_{i,2},...,\alpha_{i,N}\right)\in\left(-1,1\right)^{N\times N}$ and $\bm{\tau}=\operatorname{diag}\left(\tau_1,\tau_2,...,\tau_N\right)\in\left(0,1\right)^{N\times N}$ are the learnable parameters. The variable $a$ represents the intensity of residual diffusion across the global spatiotemporal graph originating from a particular node, indicated by its magnitude. The sign of $a$ determines whether the correlation between nodes is positive or negative, thereby mitigating the risk of weak regression caused by an overly strong assumption of homogeneity for anomalous data nodes \cite{xu2021anomaly,kim2022graph}. The parameter $\bm{\tau}$ plays a crucial role in determining the overall magnitude of the residual impact.

We denote the residual propagation layer as $\mathcal{F}_{RP}$. Let $L$ denote the number of residual propagation layers. Then, the smoothed residual $\tilde{\mathbf{Z}}_{res}=\mathcal{F}_{RP}(\mathbf{Z}_{res})$ propagated along the spatiotemporal structure can be formulated as
\begin{align}\label{L-equa22}
   \mathcal{F}_{RP}(\mathbf{Z}_{res}) = {\left[\bm{\tau}{\left(\mathbf{I}_N+\frac{1}{K}\sum_{i=1}^K\bm{\alpha}_i\mathcal{K}_i\right)}\times_2\mathbf{Z}_{res}\right]}^L.
\end{align}

\noindent\ding{184} \textbf{Prediction Correction.} We employ two layers of MLP with the GELU activation function as a residual decoder $\mathcal{F}_{R}$ to generate a prediction correction term $\mathbf{y}_{cor}=\mathcal{F}_{R}\left(\tilde{\mathbf{Z}}_{res}\right)$. Finally, this correction is added to the base prediction $\textbf{y}_{base}$ to yield final prediction $\mathbf{\hat{y}}$: 
\begin{equation}\label{prediction_y}
\mathbf{\hat{y}}=\mathbf{y}_{base}+\mathbf{y}_{cor}. 
\end{equation}

\subsection{Residual Propagation Kernel}\label{ap_kernel}
In this section, we will introduce the construction of different kernels $\mathcal{K}_i$. In traditional spatiotemporal graph learning, there are four widely used approaches to generate the potential relationships between nodes (i.e., adjacency kernel function): predefined kernel, diffusion kernel, adaptive kernel, and data-driven kernel. Each kernel $\mathcal{K}_i$ is a different normalized graph by the normalization operation $\mathcal{N}(\cdot)$. 

\noindent\ding{182} \textbf{Predefined Kernel.} This kernel is typically constructed based on various prior information, such as the geographical information of nodes. This kernel function remains static during the model learning process. Specifically, for traffic data, we calculate the geographical distance between nodes $\mathbf{d}_{s}\in\mathbb{R}^{N\times N}$ \cite{li2017diffusion,wu2019graph,yu2017spatio}, then we construct the adjacency matrix kernel in the following manner:
\begin{align}
&\mathbf{A}_{s}\coloneqq e^{-\frac{\mathbf{d}_{s}}{\sigma^2}}\odot\mathbb{I}_{\left\{\mathbf{d}_{s}<-\sigma^2\ln{\varepsilon}|\varepsilon\in\left(0,1\right)\right\}},\\
&\mathcal{K}_i=\mathcal{N}(\mathbf{A}_{s})=\mathbf{A}_{s}\mathbf{D}_{s}^{-1},
\end{align}
with degree matrix $\mathbf{D}_s$ and diagonalization operator $\operatorname{diag}$. $\mathbb{I}$ is indicator function\footnote{The indicator function can also be replaced by some transformation of Heaviside step function \cite{weisstein2002heaviside}.}and hyperparameter $\varepsilon\in\left(0,1\right)$ filters through an extremely weak correlation to ease the burden of training. $\sigma$ is the standard deviation of $\mathbf{d}_{s}$. $\odot$ is Hadamard Product. And for atmosphere data, we calculate the geographic adjacency matrix based on longitude-latitude geodesic distance matrix $\mathbf{d}_{geo}\in\mathbb{R}^{N\times N}$ \cite{wang2014multisource}~and relative altitude matrix $\mathbf{h}_{alt}\in\mathbb{R}^{N\times N}$ \cite{wang2020pm2}~if existing as follows, 
\begin{align}
&\mathbf{A}_{geo}\coloneqq\mathbb{I}_{\left\{\mathbf{d}_{geo}<\varepsilon|\varepsilon>0\right\}}\odot\mathbb{I}_{\left\{\mathbf{h}_{alt}<\xi|\xi>0\right\}},\\
&\mathcal{K}_i=\mathcal{N}(\mathbf{A}_{geo})=\mathbf{D}_{geo}^{-1/2}\mathbf{A}_{geo}\mathbf{D}_{geo}^{-1/2},
\end{align}
where the altitude satisfies, 
\begin{align}
    \mathbf{h}_{alt}\left[u,v\right]\coloneqq \sup_{\lambda\in\left(0,1\right)}\left\{h_{\lambda u+\left(1-\lambda\right)v}\right\}-\max\left\{h_u,h_v\right\}.
\end{align}

\noindent\ding{183} \textbf{Diffusion Kernel.} The diffusion kernel represents a diffusion process where information is assumed to transfer from one node to its neighboring nodes with certain transition probabilities. This concept has a strong analogy in spatiotemporal graph domains, such as the traffic flow between nodes, which can be viewed as a diffusion process. Specifically, it is generally obtained in the following ways: 
\begin{align}
&\mathbf{A}_{double}\coloneqq \left[\mathbf{A}_{road},\mathbf{A}_{road}^\top\right],\\
&\mathcal{K}_i=\mathcal{N}\left(\mathbf{A}_{double}\right)=\mathbf{A}_{double}\mathbf{D}_{double}^{-1}.
\end{align}

\noindent\ding{184} \textbf{Adaptive Kernel.} The adaptive kernel is generated with two learnable node embeddings that can capture more complex node features from the data \cite{wu2019graph,shao2022decoupled,bai2020adaptive}, which can be computed as:
\begin{align}    &\mathbf{A}_{adp}\coloneqq\operatorname{ReLU}\left(\mathbf{E}_1\mathbf{E}_2^\top\right),\\
\mathcal{N}(&\mathbf{A}_{adp})=\operatorname{Softmax}\left(\mathbf{A}_{adp}-\operatorname{diag}\left(\mathbf{A}_{adp}\right)\right),
\end{align}
where $\mathbf{E}_1, \mathbf{E}_2\in\mathbb{R}^{N\times d_{adp}}$ are two learning node embeddings.

\noindent\ding{185} \textbf{Data-driven Kernel.} This kernel is generated by a complex neural network, notably using the Transformer architecture \cite{shao2022decoupled,jiang2023pdformer}. Following these inspirations, we use Transformer to compute this kind of kernel: 
\begin{align}
\mathbf{A}_{att}\coloneqq\frac{\operatorname{MLP}_1\left(\mathbf{Z}_{res}^{\left(t\right)}\right)\operatorname{MLP}_2\left(\mathbf{Z}_{res}^{\left(t\right)}\right)^\top}{\sqrt{d_{hid}}},\\
\mathcal{N}\left(\mathbf{A}_{att}\right)=\operatorname{Softmax}\left(\mathbf{A}_{att}-\operatorname{diag}\left(\mathbf{A}_{att}\right)\right),
\end{align}
where $\mathbf{Z}_{res}^{\left(t\right)}$ is the residual representation of $t$-th time step.

\begin{algorithm}
\caption{ReLearner for spatiotemporal prediction}
\label{algorithm1}
\KwIn{Input data $\mathbf{x}\in\mathbb{R}^{T\times N\times f}$~\tcp*{\textcolor{red}{No label required}}}
\KwOut{Future label $\hat{\mathbf{y}}\in\mathbb{R}^{T_P\times N\times f}$}

$\mathbf{Z}_{e}\leftarrow\mathcal{F}_E\left(\mathbf{x}\right)$\tcp*{Input representation}

\textbf{\# Forward spatiotemporal learning}\;

$\mathbf{Z}_{h}\leftarrow \mathcal{F}_{ST}\left(\mathbf{Z}_{e}\right)$\tcp*{Label representation learning}
$\mathbf{y}_{base}\leftarrow \mathcal{F}_{ST}\left(\mathbf{Z}_{h}\right) $\tcp*{\textcolor{commentblue}{Base prediction}}
\textbf{\# Backward residual correction}\;
$\mathbf{Z}_{res}\leftarrow\mathcal{F}_{ST}\left(\mathbf{Z}_{e}-\text{MLP}\left(\mathbf{Z}_{h}\right)\right)$\tcp*{Residual learning}

$   \tilde{\mathbf{Z}}_{res} = {\left[\bm{\tau}{\left(\mathbf{I}_N+\frac{1}{K}\sum_{i=1}^K\bm{\alpha}_i\mathcal{K}_i\right)}\times_2\mathbf{Z}_{res}\right]}^L$\tcp*{Residual propagation}

$\mathbf{y}_{corr}\leftarrow\mathcal{F}_{R}\left(\tilde{\mathbf{Z}}_{res}\right)$\tcp*{{Correction prediction}}
\textbf{\# Final prediction}\;
$\hat{\mathbf{y}}\leftarrow\mathbf{y}_{base}+\mathbf{y}_{corr}$\tcp*{{Final prediction}}
\end{algorithm}

\section{Experiment}\label{sec:exp}
In this section, we evaluate the effectiveness and generality of the proposed generic ReLeaner modules across different real-world datasets and over multiple advanced baselines and formulate seven research questions (RQ) as follows:

\begin{itemize}[label=\ding{110}]
\item \textbf{RQ1:} Does the integration of ReLearner consistently improve the predictive performance of existing STNNs across different tasks and datasets?
\item \textbf{RQ2:} Can ReLearner effectively mitigate spatiotemporal input–label deviation by explicitly modeling and correcting spatiotemporal residuals?
\item \textbf{RQ3:} How does ReLearner influence the training procedure of STNNs, particularly in terms of convergence stability and optimization efficiency?
\item \textbf{RQ4:} What are the computational cost and architecture efficiency of ReLearner, and do its performance gains stem from architectural innovation rather than increased model complexity?
\item \textbf{RQ5:} How sensitive is the performance of ReLearner to hyperparameter configurations, such as kernel types or residual propagation layers?
\item \textbf{RQ6:} How does ReLearner differ from temporal shifts technologies in addressing spatiotemporal input–label deviation?
\item \textbf{RQ7:} How does ReLearner differ from prior approaches, STID, in addressing spatiotemporal input–label deviation?
\end{itemize}

\subsection{Experimental Setup}

\noindent\ding{182} \textbf{Datasets.} We evaluate our approach on 11 real-world spatiotemporal datasets from two major domains: transportation and atmospheric monitoring. In the transportation domain, we employ several widely used benchmarks, including the recently introduced LargeST benchmark \cite{liu2024largest}, as well as the normal and common PEMS3-Stream \cite{chen2021trafficstream}, PeMS0\{3,4, 7,8\} \cite{song2020spatial}, and METR-LA \cite{li2017diffusion} datasets. 

LargeST is a large-scale, comprehensive traffic dataset specifically designed for spatiotemporal prediction. It collects highway speed data, recorded every 15 minutes, from up to 8,600 highway sensors in the California Performance Measurement System (PeMS). It comprises four subsets: SD, GBA, GLA, and CA. For PeMS3D-Stream, the raw 30-second traffic flow data collected by CalTrans are aggregated into 5-minute intervals, yielding a dataset of 655 sensors in California’s North Central Area. We adopt the July 2011 portion of the data for evaluation.
In the atmospheric domain, we employ the KnowAir dataset \cite{wang2020pm2}, which records four years (2015–2018) of PM2.5 concentrations from 184 monitoring stations across China. A comprehensive summary of all datasets, including node count, sampling frequency, and temporal coverage, is provided in Table~\ref{Dataset summary}. 

\begin{table}[htbp]
  \setlength{\arrayrulewidth}{1.0pt}
  \renewcommand{\arraystretch}{1.2}
  \centering
  \caption{The details of traffic datasets used in this paper.}
\resizebox{\linewidth}{!}{
    \begin{tabular}{lcccc}
    \hline
    \rowcolor[rgb]{.867,.875,  .91}\textbf{Dataset} & \textbf{Nodes} & \textbf{Edges} & \textbf{Time Range} & \textbf{Frames} \\
    \hline
    LargeST-SD    & 716   & 17,319 & 01/01/2019 - 31/12/2019 & 525,888 \\
    \rowcolor[rgb]{ .949,  .949,  .949}LargeST-GBA   & 2,352 & 61,246 & 01/01/2019 - 31/12/2019 & 525,888 \\
    LargeST-GLA    & 3,834 & 98,703 & 01/01/2019 - 31/12/2019 & 525,888 \\
    \rowcolor[rgb]{ .949,  .949,  .949}LargeST-CA    & 8,600 & 201,363 & 01/01/2019 - 31/12/2019 & 525,888 \\
    \rowcolor[rgb]{ .949,  .949,  .949}PEMS03 & 358   & 546   & 09/01/2018 – 11/30/2018 & 26,208 \\
    PEMS04 & 307   & 338   & 01/01/2018 – 02/28/2018  & 16,992 \\
    \rowcolor[rgb]{ .949,  .949,  .949}PEMS08 & 170   & 276   & 07/01/2016 – 08/31/2016 & 17,856 \\
    PEMS07 & 883   & 865   & 05/01/2017 – 08/06/2017 & 28,224 \\
    \rowcolor[rgb]{ .949,  .949,  .949}METR-LA & 207   & 1,515 & 03/01/2012 – 06/27/2012 & 34,272 \\
    PEMS3-Stream &655&1,577 &07/01/2011 - 07/31/2011 &8,928\\
    \rowcolor[rgb]{ .949,  .949,  .949}KnowAir & 184   & 3,796 & 01/01/2015 - 12/31/2018 & 3,4380\\
    \hline
    \end{tabular}
    }%
  \label{Dataset summary}%
\end{table}

\noindent\ding{183} \textbf{Baselines.} To evaluate the effectiveness of ReLearner, we integrated it into 14 representative spatiotemporal forecasting models. The baselines include: \underline{\textbf{Traditional methods}} such as HL~\cite{liang2021revisiting} and LSTM; \underline{\textbf{MLP-based STNNs}} including STID~\cite{10.1145/3511808.3557702} and BigST~\cite{han2024bigst}; \underline{\textbf{GNN-based STNNs}} such as STGCN~\cite{yu2017spatio}, AGCRN~\cite{bai2020adaptive}, DGCRN~\cite{li2023dynamic}, DDGCRN~\cite{weng2023decomposition}, GC-LSTM~\cite{qi2019hybrid}, PM\textsubscript{2.5}GNN~\cite{wang2020pm2}, and GWNet~\cite{wu2019graph}; \underline{\textbf{Transformer-based STNNs}} comprising STAEFormer~\cite{liu2023spatio}, STNN~\cite{he2020stnn}, and D\textsuperscript{2}STGNN~\cite{shao2022decoupled}. For broader comparison, we also evaluated two general \underline{\textbf{spatiotemporal drift learning methods}}: RevIN~\cite{kim2021reversible} and Dish-TS~\cite{fan2023dish}.

\noindent\ding{184} \textbf{Metric}. To assess the efficacy of our framework, we employed metrics commonly utilized in spatiotemporal prediction tasks, including Mean Absolute Error (MAE), Root Mean Square Error (RMSE), and Mean Absolute Percentage Error (MAPE). Moreover, we consider specific metrics, including Critical Success Index (CSI), Probability of Detection (POD), and False Alarm Rate (FAR), to assess the performance of the system in atmospheric tasks. Let the prediction value be $\hat{\mathbf{y}}_{:,u}$ and ground truth value be $\mathbf{y}_{:,u}$ for a specific node $u$, and we choose $\varepsilon=75~\mu g/m^3$ to be the demarcation point of good air quality \cite{zhao2016evolution}. Hence the specific metrics satisfy,
\begin{align}
    \text{CSI}&=\frac{\#\left\{t\mid\mathbf{y}_{t,u}\geq\varepsilon,\hat{\mathbf{y}}_{t,u}\geq\varepsilon\right\}}{24-\#\left\{t\mid\mathbf{y}_{t,u}<\varepsilon,\hat{\mathbf{y}}_{t,u}<\varepsilon\right\}},\\
    \text{POD}&=\frac{\#\left\{t\mid\mathbf{y}_{t,u}\geq\varepsilon,\hat{\mathbf{y}}_{t,u}\geq\varepsilon\right\}}{\#\left\{t\mid\mathbf{y}_{t,u}\geq\varepsilon\right\}},\\
    \text{FAR}&=\frac{\#\left\{t\mid\mathbf{y}_{t,u}<\varepsilon,\hat{\mathbf{y}}_{t,u}\geq\varepsilon\right\}}{\#\left\{t\mid\hat{\mathbf{y}}_{t,u}\geq\varepsilon\right\}},
\end{align}
where $\#$ calculates the cardinal of the following set. It is crucial to acknowledge that smaller metrics indicate superior model performance for all metrics except CSI and POD. Conversely, the opposite is true for CSI and POD metrics.

\noindent\ding{185} \textbf{Setting.}~All datasets are split into training, validation, and test sets with a ratio of 6:2:2 along the temporal dimension. We use the AdamW optimizer \cite{loshchilov2017decoupled} with a learning rate of 0.002, and optimize the models using MAE as the loss function. All experiments are conducted on an NVIDIA A100 GPU with 40 GB memory, and the implementation is based on the PyTorch framework using Python 3.8.3. The lengths of both the input time window and the future prediction window are set to 12 for traffic datasets and 24 for atmospheric datasets. \textit{When training STNNs with ReLearner, we keep all hyperparameters of the original STNNs unchanged to ensure that the performance improvements arise solely from ReLearner.}

\begin{table*}[!t]
	\belowrulesep=0.5pt
	\aboverulesep=0.5pt
	\centering
	\caption{Prediction performance of models on traffic datasets. `$\Delta$' reports the improvement of average prediction performance during 12 time steps using ReLearner relative to base models.}
    \label{results_traffic}
	\resizebox{\textwidth}{!}{    
		\begin{tabular}{c|lccccccccc|ccc}
			\toprule
			\multicolumn{2}{c}{\multirow{2}[0]{*}{\textbf{Method}}} & \multicolumn{3}{c}{\textbf{Horizon 3}} & \multicolumn{3}{c}{\textbf{Horizon 6}} & \multicolumn{3}{c|}{\textbf{Horizon 12}} & \multicolumn{3}{c}{\textbf{$\Delta$}} \\
			\cmidrule{3-14}    \multicolumn{2}{c}{} & MAE   & RMSE  & MAPE  & MAE   & RMSE  & MAPE  & MAE   & RMSE  & MAPE  & MAE   & RMSE  & MAPE \\
			\midrule
			\midrule
			\multirow{22}[0]{*}{\begin{sideways}LargeST-SD\end{sideways}} & LSTM  & 19.13  & 30.80  & 11.62  & 26.07  & 41.34  & 16.32  & 37.87  & 59.37  & 25.08  & -     & -     & - \\
			& LSTM + ReLearner & \textbf{17.55} & \textbf{27.97} & \textbf{11.42} & \textbf{21.71} & \textbf{34.64} & \textbf{14.68} & \textbf{26.91} & \textbf{44.06} & \textbf{19.44} & \cellcolor{lv}+19.50\% & \cellcolor{lv}+18.20\% & \cellcolor{lv}+13.37\% \\
			\cmidrule{2-14}          & STID  & 15.39  & 25.71  & 9.90  & 18.05  & 30.53  & 12.02  & 22.01  & 39.06  & 15.35  & -     & -     & - \\
			& STID + ReLearner & \textbf{14.66} & \textbf{24.71} & \textbf{9.45} & \textbf{17.08} & \textbf{29.03} & \textbf{11.20} & \textbf{20.83} & \textbf{35.64} & \textbf{14.31} & \cellcolor{lv}+5.11\% & \cellcolor{lv}+5.95\% & \cellcolor{lv}+6.06\% \\
			\cmidrule{2-14}          & STAEformer & 15.68  & 25.71  & 10.65  & 18.33  & 30.42  & 12.66  & 22.77  & 38.64  & 15.99  & -     & -     & - \\
			& STAEformer + ReLearner & \textbf{15.48} & \textbf{25.92} & \textbf{10.14} & \textbf{17.91} & \textbf{30.31} & \textbf{11.86} & \textbf{21.61} & \textbf{37.13} & \textbf{14.96} & \cellcolor{lv}+3.09\% & \cellcolor{lv}+1.40\% & \cellcolor{lv}+4.89\% \\
			\cmidrule{2-14}          & STGCN  & 17.37  & 29.91  & 12.36  & 19.29  & 33.36  & 13.39  & 22.99  & 40.28  & 15.80  & -     & -     & - \\
			& STGCN + ReLearner & \textbf{16.05} & \textbf{27.39} & \textbf{10.80} & \textbf{18.30} & \textbf{31.45} & \textbf{12.21} & \textbf{22.17} & \textbf{39.22} & \textbf{15.12} & \cellcolor{lv}+5.73\% & \cellcolor{lv}+5.77\% & \cellcolor{lv}+8.79\% \\
			\cmidrule{2-14}          & STTN  & 18.11  & 28.92  & 11.33  & 21.26  & 34.33  & 13.30  & 25.78  & 41.43  & 17.06  & -     & -     & - \\
			& STTN + ReLearner & \textbf{15.94} & \textbf{25.76} & \textbf{10.66} & \textbf{18.49} & \textbf{30.26} & \textbf{12.33} & \textbf{22.54} & \textbf{38.58} & \textbf{16.10} & \cellcolor{lv}+12.18\% & \cellcolor{lv}+9.85\% & \cellcolor{lv}+6.12\% \\
			\cmidrule{2-14}          & ASTGCN & 20.23  & 32.17  & 13.09  & 25.94  & 40.54  & 17.13  & 32.34  & 50.86  & 22.23  & -     & -     & - \\
			& ASTGCN + ReLearner & \textbf{17.41} & \textbf{28.24} & \textbf{11.17} & \textbf{20.92} & \textbf{34.05} & \textbf{13.59} & \textbf{24.80} & \textbf{40.51} & \textbf{17.22} & \cellcolor{lv}+18.74\% & \cellcolor{lv}+16.00\% & \cellcolor{lv}+21.18\% \\
			\cmidrule{2-14}          & AGCRN & 15.57  & 28.49  & 11.39  & 17.66  & 31.44  & 12.86  & 21.40  & 40.44  & 16.35  & -     & -     & - \\
			& AGCRN + ReLearner & \textbf{15.14} & \textbf{25.49} & \textbf{10.16} & \textbf{17.39} & \textbf{29.70} & \textbf{11.70} & \textbf{20.95} & \textbf{29.96} & \textbf{14.91} & \cellcolor{lv}+2.03\% & \cellcolor{lv}+7.39\% & \cellcolor{lv}+9.57\% \\
			\cmidrule{2-14}          & DGCRN & 15.83  & 28.48  & 12.60  & 20.50  & 33.24  & 14.08  & 24.16  & 40.67  & 16.68  & -     & -     & - \\
			& DGCRN + ReLearner & \textbf{15.28} & \textbf{25.45} & \textbf{10.98} & \textbf{17.33} & \textbf{29.27} & \textbf{11.55} & \textbf{21.20} & \textbf{36.57} & \textbf{14.55} & \cellcolor{lv}+1.14\% & \cellcolor{lv}+0.07\% & \cellcolor{lv}+4.30\% \\
			\cmidrule{2-14}          & DDGCRN & 15.64  & 29.23  & 11.12  & 18.34  & 33.19  & 12.82  & 22.79  & 40.97  & 16.46  & -     & -     & - \\
			& DDGCRN + ReLearner & \textbf{15.59} & \textbf{28.00} & \textbf{11.10} & \textbf{18.13} & \textbf{31.53} & \textbf{11.99} & \textbf{22.10} & \textbf{39.17} & \textbf{15.02} & \cellcolor{lv}+1.73\% & \cellcolor{lv}+3.80\% & \cellcolor{lv}+5.95\% \\
			\cmidrule{2-14}          & D$^2$STGNN & 14.93  & 25.29  & 10.37  & 17.40  & 29.69  & 12.16  & 21.31  & 36.30  & 14.99  & -     & -     & - \\
			& D$^2$STGNN + ReLearner & \textbf{14.83} & \textbf{24.68} & \textbf{9.79} & \textbf{17.23} & \textbf{28.72} & \textbf{11.17} & \textbf{20.61} & \textbf{34.42} & \textbf{13.58} & \cellcolor{lv}+1.41\% & \cellcolor{lv}+3.55\% & \cellcolor{lv}+8.11\% \\
			\cmidrule{2-14}          & BigST & 15.83  & 26.04  & 11.38  & 18.17  & 31.13  & 13.12  & 22.92  & 39.63  & 16.34  & -     & -     & - \\
			& BigST + ReLearner & \textbf{14.39} & \textbf{24.27} & \textbf{10.24} & \textbf{17.56} & \textbf{29.09} & \textbf{11.78} & \textbf{21.01} & \textbf{36.01} & \textbf{14.88} & \cellcolor{lv}+8.63\% & \cellcolor{lv}+6.52\% & \cellcolor{lv}+10.31\% \\
			\midrule
			\midrule
			\multirow{22}[0]{*}{\begin{sideways}LargeST-GBA\end{sideways}} & LSTM  & 20.21  & 33.22  & 15.14  & 27.28  & 43.34  & 23.08  & 38.55  & 60.13  & 36.68  & -     & -     & - \\
			& LSTM + ReLearner & \textbf{17.67} & \textbf{31.24} & \textbf{15.14} & \textbf{24.47} & \textbf{38.37} & \textbf{22.09} & \textbf{31.77} & \textbf{49.32} & \textbf{32.92} & \cellcolor{lv}+11.89\% & \cellcolor{lv}+12.70\% & \cellcolor{lv}+6.34\% \\
			\cmidrule{2-14}          & STID  & 17.80  & 29.56  & 14.32  & 21.04  & 34.76  & 17.28  & 25.23  & 42.22  & 21.48  & -     & -     & - \\
			& STID + ReLearner & \textbf{17.43} & \textbf{29.35} & \textbf{13.37} & \textbf{20.43} & \textbf{34.19} & \textbf{15.94} & \textbf{24.35} & \textbf{40.90} & \textbf{19.92} & \cellcolor{lv}+2.74\% & \cellcolor{lv}+1.88\% & \cellcolor{lv}+7.08\% \\
			\cmidrule{2-14}          & STAEformer & 18.55  & 29.94  & 14.99  & 21.69  & 34.65  & 16.87  & 26.42  & 41.50  & 21.31  & -     & -     & - \\
			& STAEformer + ReLearner & \textbf{18.05} & \textbf{29.40} & \textbf{14.43} & \textbf{21.18} & \textbf{34.04} & \textbf{16.17} & \textbf{25.76} & \textbf{41.21} & \textbf{20.80} & \cellcolor{lv}+3.46\% & \cellcolor{lv}+1.16\% & \cellcolor{lv}+1.93\% \\
			\cmidrule{2-14}          & STGCN  & 20.47  & 33.85  & 15.26  & 22.75  & 37.49  & 17.03  & 25.51  & 42.13  & 19.80  & -     & -     & - \\
			& STGCN + ReLearner & \textbf{19.38} & \textbf{32.14} & \textbf{14.52} & \textbf{22.16} & \textbf{36.65} & \textbf{16.73} & \textbf{25.61} & \textbf{42.62} & \textbf{20.04} & \cellcolor{lv}+2.31\% & \cellcolor{lv}+2.02\% & \cellcolor{lv}+3.25\% \\
			\cmidrule{2-14}          & STTN  & 18.92  & 30.48  & 15.25  & 22.31  & 35.50  & 18.88  & 26.59  & 42.58  & 23.35  & -     & -     & - \\
			& STTN + ReLearner & \textbf{18.70} & \textbf{29.98} & \textbf{15.07} & \textbf{21.99} & \textbf{35.02} & \textbf{17.89} & \textbf{26.05} & \textbf{42.06} & \textbf{22.49} & \cellcolor{lv}+1.73\% & \cellcolor{lv}+1.45\% & \cellcolor{lv}+4.27\% \\
			\cmidrule{2-14}          & ASTGCN & 21.53  & 34.07  & 17.44  & 26.31  & 40.36  & 24.71  & 34.00  & 52.97  & 30.15  & -     & -     & - \\
			& ASTGCN + ReLearner & \textbf{19.91} & \textbf{31.60} & \textbf{16.29} & \textbf{24.73} & \textbf{38.52} & \textbf{21.18} & \textbf{30.51} & \textbf{47.44} & \textbf{27.73} & \cellcolor{lv}+7.01\% & \cellcolor{lv}+5.35\% & \cellcolor{lv}+9.04\% \\
			\cmidrule{2-14}          & AGCRN & 18.04  & 30.11  & 13.99  & 20.79  & 34.29  & 16.33  & 24.28  & 39.65  & 20.21  & -     & -     & - \\
			& AGCRN + ReLearner & \textbf{16.80} & \textbf{28.56} & \textbf{12.52} & \textbf{20.10} & \textbf{33.45} & \textbf{14.98} & \textbf{23.73} & \textbf{39.10} & \textbf{18.79} & \cellcolor{lv}+2.00\% & \cellcolor{lv}+1.80\% & \cellcolor{lv}+6.95\% \\
			\cmidrule{2-14}          & DGCRN & 18.02  & 28.97  & 15.23  & 21.09  & 33.88  & 18.13  & 25.86  & 41.02  & 23.66  & -     & -     & - \\
			& DGCRN + ReLearner & \textbf{17.86} & \textbf{28.43} & \textbf{15.10} & \textbf{20.89} & \textbf{32.76} & \textbf{18.03} & \textbf{25.69} & \textbf{40.89} & \textbf{22.35} & \cellcolor{lv}+1.38\% & \cellcolor{lv}+0.57\% & \cellcolor{lv}+4.67\% \\
			\cmidrule{2-14}          & DDGCRN & 17.86  & 29.11  & 15.26  & 21.05  & 33.86  & 18.07  & 25.62  & 41.15  & 23.55  & -     & -     & - \\
			& DDGCRN + ReLearner & \textbf{17.47} & \textbf{28.22} & \textbf{15.24} & \textbf{20.72} & \textbf{32.74} & \textbf{18.00} & \textbf{25.57} & \textbf{40.81} & \textbf{22.19} & \cellcolor{lv}+2.64\% & \cellcolor{lv}+0.81\% & \cellcolor{lv}+5.49\% \\
			\cmidrule{2-14}          & D$^2$STGNN & 17.23  & 29.91  & 12.22  & 20.50  & 34.79  & 14.96  & 25.13  & 41.80  & 19.67  & -     & -     & - \\
			& D$^2$STGNN + ReLearner & \textbf{16.83} & \textbf{29.27} & \textbf{11.97} & \textbf{20.32} & \textbf{34.35} & \textbf{14.56} & \textbf{24.81} & \textbf{41.21} & \textbf{19.27} & \cellcolor{lv}+1.03\% & \cellcolor{lv}+1.19\% & \cellcolor{lv}+5.13\% \\
			\cmidrule{2-14}          & BigST & 18.29  & 29.79  & 15.18  & 21.98  & 35.39  & 18.91  & 26.74  & 43.48  & 23.91  & -     & -     & - \\
			& BigST + ReLearner & \textbf{17.37} & \textbf{29.26} & \textbf{13.78} & \textbf{20.88} & \textbf{33.87} & \textbf{15.89} & \textbf{24.79} & \textbf{41.52} & \textbf{20.32} & \cellcolor{lv}+17.36\% & \cellcolor{lv}+3.48\% & \cellcolor{lv}+15.50\% \\
			\bottomrule
	\end{tabular}}
\end{table*}%

\begin{table*}[!t]

  \centering
  \caption{Predictive performance of the model on KnowAir dataset. `Average improvement' reports the improvement of average prediction performance during 12 time steps using ReLearner relative to base models}\resizebox{\linewidth}{!}{   
    \begin{tabular}{clccccccccccccc}
    \toprule
    \multirow{2}[0]{*}{\textbf{Dataset}} & \multicolumn{2}{c}{\textbf{Method}} & \multicolumn{6}{c}{\textbf{Average}}                   & \multicolumn{6}{c}{\textbf{Average relative  improvement}} \\
\cmidrule{2-15}          & Basemodel & ReLearner & MAE   & RMSE  & MAPE & CSI & POD & FAR & MAE & RMSE & MAPE & CSI & POD & FAR \\
    \midrule
    \multirow{14}[14]{*}{\begin{sideways}KnowAir\end{sideways}} & \multirow{2}[0]{*}{LSTM} & -     & 20.88 & 35.97 & 39.80  & 72.03 & 88.68 & 20.68 & -     & -     & -     & -     & -     & - \\
          &       & \ding{52} & \textbf{20.63} & \textbf{35.47} & \textbf{39.57} & \textbf{72.55} & \textbf{89.46} & \textbf{20.66} & \cellcolor{lv}+1.20\%   & \cellcolor{lv}+1.37\%  & \cellcolor{lv}+0.58\%  & \cellcolor{lv}+0.72\%  & \cellcolor{lv}+0.88\%  & \cellcolor{lv}+0.10\%\\
\cmidrule{2-15}          & \multirow{2}[0]{*}{GC-LSTM} & -     & 20.91 & 36.00    & 39.32 & 72.30  & 88.94 & 20.56 & -     & -     & -     & -     & -     & - \\
          &       & \ding{52} & \textbf{19.49} & \textbf{33.89} & \textbf{36.14} & \textbf{73.71} & \textbf{88.36} & \textbf{18.36} & \cellcolor{lv}+6.79\%  & \cellcolor{lv}+5.86\%  & \cellcolor{lv}+8.09\%  & \cellcolor{lv}+1.95\%  & \cellcolor{lv}+0.65\% & \cellcolor{lv}+10.70\% \\
\cmidrule{2-15}          & \multirow{2}[0]{*}{STID} & -     & 17.50  & 31.63 & 33.04 & 76.04 & 88.75 & 15.85 & -     & -     & -     & -     & -     & - \\
          &       & \ding{52} & \textbf{17.47} & \textbf{30.88} & \textbf{32.07} & \textbf{76.60} & \textbf{92.13} & \textbf{15.33} & \cellcolor{lv}+0.17\%  & \cellcolor{lv}+2.37\%  & \cellcolor{lv}+2.94\%  & \cellcolor{lv}+0.74\%  & \cellcolor{lv}+3.81\%  & \cellcolor{lv}+3.28\% \\
\cmidrule{2-15}          & \multirow{2}[0]{*}{STAEFormer} & -     & 18.28 & 31.33 & 36.42 & 75.68 & 90.89 & 19.67 & -     & -     & -     & -     & -     & - \\
          &       & \ding{52} & \textbf{17.56} & \textbf{30.41} & \textbf{30.41} & \textbf{76.54} & \textbf{92.28} & \textbf{18.22} & \cellcolor{lv}+3.94\%  & \cellcolor{lv}+2.94\%  & \cellcolor{lv}+16.50\%  & \cellcolor{lv}+1.14\%  & \cellcolor{lv}+1.53  & \cellcolor{lv}+7.37\% \\
\cmidrule{2-15}          & \multirow{2}[0]{*}{AGCRN} & -     & 19.83 & 34.25 & 39.00    & 73.24 & 90.14 & 20.38 & -     & -     & -     & -     & -     & - \\
          &       & \ding{52} & \textbf{19.61} & \textbf{33.41} & \textbf{38.62} & \textbf{73.32} & \textbf{90.67} & \textbf{19.49} & \cellcolor{lv}+1.11\%  & \cellcolor{lv}+2.45\%  & \cellcolor{lv}+0.97\%  & \cellcolor{lv}+0.11\%  & \cellcolor{lv}+0.59\%  & \cellcolor{lv}+4.37\% \\
\cmidrule{2-15}          & \multirow{2}[0]{*}{DDGCRN} & -     & 19.68 & 34.60  & 36.36 & 72.94 & 86.29 & 18.49 & -     & -     & -     & -     & -     & - \\
          &       & \ding{52} & \textbf{19.00} & \textbf{34.22} & \textbf{36.28} & \textbf{72.91} & \textbf{88.54} & \textbf{17.86} & \cellcolor{lv}+3.46\%  & \cellcolor{lv}+1.01\%  & \cellcolor{lv}+0.22  & \cellcolor{lv}+0.04\% & \cellcolor{lv}+2.61\%  & \cellcolor{lv}+3.41\% \\
\cmidrule{2-15}          & \multirow{2}[0]{*}{PM$_{2.5}$GNN} & -     & 20.28 & 34.97 & 38.25 & 72.27 & 88.43 & 20.18 & -     & -     & -     & -     & -     & - \\
          &       & \ding{52} & \textbf{18.99} & \textbf{32.97} & \textbf{36.75} & \textbf{74.35} & \textbf{89.45} & \textbf{18.50} & \cellcolor{lv}+6.36\%  & \cellcolor{lv}+5.72\%  & \cellcolor{lv}+3.92\%  & \cellcolor{lv}+2.88\%  & \cellcolor{lv}+1.15\%  & \cellcolor{lv}+8.33\% \\
    \bottomrule
    \end{tabular}}
  \label{results_air}%
\end{table*}%

\begin{table*}[t]
	\belowrulesep=0.5pt
	\aboverulesep=0.1pt
	\centering
	\caption{Prediction performance of models on 6 small scale and 2 large scale traffic datasets. `$\Delta$' reports the improvement of average prediction performance during 12 time steps using ReLearner relative to base models.}
	\resizebox{\linewidth}{!}{\begin{tabular}{cccc|ccc|ccc|ccc}
			\toprule
			\rowcolor[rgb]{ .949,  .949,  .949}\textbf{Dataset} & \multicolumn{3}{c|}{\textbf{PeMS03}} & \multicolumn{3}{c|}{\textbf{PeMS04}} & \multicolumn{3}{c|}{\textbf{PeMS08}} & \multicolumn{3}{c}{\textbf{METR-LA}} \\
			\midrule
			Method & MAE   & RMSE  & MAPE  & MAE   & RMSE  & MAPE  & MAE   & RMSE  & MAPE  & MAE   & RMSE  & MAPE \\
			\midrule
			LSTM  & 21.33 & 35.11 & 23.33 & 23.81 & 36.62 & 18.12 & 21.31 & 32.10 & 17.47 & 3.55  & 7.10  & 10.18 \\
			+ ReLearner   & \textbf{16.21} & \textbf{27.44} & \textbf{15.83} & \textbf{21.37} & \textbf{33.65} & \textbf{14.30} & \textbf{16.27} & \textbf{26.04} & \textbf{10.43} & \textbf{3.27} & \textbf{6.42} & \textbf{9.28} \\
			\rowcolor{lv}$\Delta$  & +24.00\% & +21.84\% & +32.14\% & +10.24\% & +8.11\% & +21.08\% & +23.65\% & +18.87\% & +40.29\% & +7.88\% & +9.57\% & +8.84\% \\
			\midrule
			STID  & 15.36 & 25.97 & 16.20 & 18.60 & 30.14 & 12.28 & 14.21 & 23.43 & 9.28  & 3.22  & 6.58  & 9.16 \\
			+ ReLearner   & \textbf{15.08} & \textbf{25.75} & \textbf{15.49} & \textbf{17.89} & \textbf{29.66} & \textbf{11.99} & \textbf{13.80} & \textbf{23.09} & \textbf{9.18} & \textbf{3.06} & \textbf{6.19} & \textbf{8.54} \\
			\rowcolor{lv}$\Delta$  & +1.82\% & +0.85\% & +4.38\% & +3.81\% & +1.59\% & +2.36\% & +2.88\% & +1.45\% & +1.07\% & +4.96\% & +5.92\% & +6.76\% \\
			\midrule
			STAEformer & 15.45 & 27.39 & 15.08 & 18.17 & 29.99 & 11.92 & 13.59 & 23.93 & 8.83  & 3.02  & 6.07  & 8.35 \\
			+ ReLearner   & \textbf{15.13} & \textbf{26.80} & \textbf{14.72} & \textbf{18.02} & \textbf{28.56} & \textbf{11.08} & \textbf{13.36} & \textbf{23.35} & \textbf{8.43} & \textbf{2.99} & \textbf{6.06} & \textbf{8.16} \\
			\rowcolor{lv}$\Delta$  & +2.07\% & +2.15\% & +2.38\% & +0.82\% & +4.76\% & +7.04\% & +1.69\% & +2.42\% & +4.53\% & +1.00\% & +0.16\% & +2.27\% \\
			\midrule
			STGCN & 17.47 & 28.81 & 17.08 & 20.01 & 31.82 & 13.32 & 15.69 & 25.19 & 10.31 & 3.11  & 6.26  & 8.60 \\
			+ ReLearner   & \textbf{16.74} & \textbf{28.12} & \textbf{16.89} & \textbf{19.13} & \textbf{30.83} & \textbf{13.15} & \textbf{15.07} & \textbf{24.50} & \textbf{9.91} & \textbf{3.03} & \textbf{6.10} & \textbf{8.24} \\
			\rowcolor{lv}$\Delta$  & +4.18\% & +2.39\% & +1.11\% & +4.39\% & +3.11\% & +1.27\% & +3.95\% & +2.73\% & +3.87\% & +2.57\% & +2.55\% & +4.18\% \\
			\midrule
			AGCRN & 16.06 & 28.49 & 15.85 & 19.83 & 32.26 & 12.97 & 15.59 & 25.07 & 10.19 & 3.15  & 6.38  & 8.81 \\
			+ ReLearner   & \textbf{15.27} & \textbf{26.91} & \textbf{14.71} & \textbf{18.85} & \textbf{30.95} & \textbf{12.41} & \textbf{15.15} & \textbf{24.89} & \textbf{10.07} & \textbf{3.14} & \textbf{6.31} & \textbf{8.71} \\
			\rowcolor{lv}$\Delta$  & +4.91\% & +5.54\% & +7.19\% & +4.94\% & +4.06\% & +4.31\% & +2.82\% & +0.71\% & +1.17\% & +0.31\% & +1.09\% & +1.13\% \\
			\midrule
			GWNet & 16.77 & 27.57 & 16.11 & 21.79 & 33.79 & 14.85 & 18.03 & 27.86 & 9.41  & 3.03  & 6.04  & 8.22 \\
			+ ReLearner   & \textbf{16.41} & \textbf{27.00} & \textbf{15.20} & \textbf{21.11} & \textbf{32.87} & \textbf{14.37} & \textbf{17.96} & \textbf{27.56} & \textbf{9.10} & \textbf{3.01} & \textbf{6.02} & \textbf{8.21} \\
			\rowcolor{lv}$\Delta$  & +2.14\% & +2.06\% & +5.64\% & +3.12\% & +2.72\% & +3.23\% & +0.38\% & +1.07\% & +3.29\% & +0.66\% & +0.33\% & +0.12\% \\
			\midrule
			D$^2$STGNN & 14.62 & 25.09 & 14.23 & 18.53 & 30.68 & 12.17 & 14.36 & 23.76 & 9.37  & 3.01  & 6.05  & 8.41 \\
			+ ReLearner   & \textbf{14.51} & \textbf{24.54} & \textbf{13.92} & \textbf{18.22} & \textbf{30.17} & \textbf{12.00} & \textbf{13.77} & \textbf{23.35} & \textbf{8.99} & \textbf{2.94} & \textbf{6.02} & \textbf{8.12} \\
			\rowcolor{lv}$\Delta$  & +0.75\% & +1.35\% & +2.14\% & +1.67\% & +1.66\% & +1.40\% & +4.10\% & +1.73\% & +4.05\% & +1.00\% & +0.50\% & +3.45\% \\
			\midrule[\heavyrulewidth]
			\rowcolor[rgb]{ .949,  .949,  .949}\multicolumn{1}{c}{\textbf{Dataset}} & \multicolumn{3}{c|}{\textbf{PeMS07}} & \multicolumn{3}{c|}{\textbf{PEMS3-Stream}} & \multicolumn{3}{c|}{\textbf{GLA}} & \multicolumn{3}{c}{\textbf{CA}} \\
			\midrule
			Method & MAE   & RMSE  & MAPE  & MAE   & RMSE  & MAPE  & MAE   & RMSE  & MAPE  & MAE   & RMSE  & MAPE \\
			\midrule
			STNorm & 20.56 & 34.88 & 8.63  & 11.92&	18.56&	15.63 & 21.31& 	34.53& 	14.06 & 19.30 & 31.98 & 14.02 \\
			+ ReLearner   & \textbf{19.88} & \textbf{33.25} & \textbf{8.30} & \textbf{11.63} & \textbf{18.14} & \textbf{15.26} & \textbf{21.09} & \textbf{33.80} & \textbf{12.71} & \textbf{19.03} & \textbf{31.35} & \textbf{13.23} \\
			\rowcolor{lv}$\Delta$  & +3.30\% & +4.67\% & +3.82\% & +2.43\% & +2.37\% & +1.29\% & +1.03\% & +2.11\% & +9.60\% & +1.40\% & +1.94\% & +5.62\% \\
			\midrule
			GWNet & 24.55 & 38.36 & 10.15 & 12.44&	18.98	&16.78 &  21.21 &33.63 & 13.73 & 21.74 & 34.22 & 17.41 \\
			+ ReLearner   & \textbf{23.46} & \textbf{37.65} & \textbf{9.86} & \textbf{11.59} & \textbf{17.81} & \textbf{15.33} & \textbf{20.65} & \textbf{32.97} & \textbf{13.42} &   \textbf{19.97}   &   \textbf{32.26}    &   \textbf{14.28}  \\
			\rowcolor{lv}$\Delta$  & +4.44\% & +1.85\% & +2.85\% & +3.74\% & +3.56\% & +8.64\% & +3.60\% & +1.96\% & +2.25\% & +8.14\%   & +3.37\%   & +3.13\% \\
			\midrule
			STID  & 19.52 & 32.90 & 8.27  & 12.58&	19.36&	16.21&21.69&35.20 & 14.39 & 19.10 & 32.00 & 14.73 \\
			+ ReLearner   & \textbf{19.17} & \textbf{32.45} & \textbf{8.18} & \textbf{11.72} & \textbf{17.96} & \textbf{15.57} & \textbf{21.46} & \textbf{34.57} & \textbf{13.61} & \textbf{18.50} & \textbf{30.92} & \textbf{13.61} \\
			\rowcolor{lv}$\Delta$  & +1.79\% & +1.37\% & +1.09\% & +6.83\% & +7.23\% & +3.95\% & +1.06\% & +1.78\% & +5.28\% & +3.14\% & +3.38\% & +7.60\% \\
			\midrule
			STGCN & 21.62 & 34.89 & 13.99 & 13.42&20.27&17.73 & 22.62&38.71&14.12
			& 21.36 & 36.42 & 16.55 \\
			+ ReLearner   & \textbf{20.47} & \textbf{32.76} & \textbf{12.94} & \textbf{12.07} & \textbf{19.14} & \textbf{16.56} & \textbf{21.51} & \textbf{37.15} & \textbf{13.19} & \textbf{19.85} & \textbf{34.30} & \textbf{14.43} \\
			\rowcolor{lv}$\Delta$  & +10.05\% & +5.57\% & +6.60\% & +5.73\% & +4.91\% & +4.03\% & +6.59\% & +2.02\% & +0.35\% & +7.06\% & +5.82\% & +12.80\% \\
			\midrule
			D$^2$STGNN & 19.77 & 33.08 & 8.40  & 12.98&20.36&17.24& \multicolumn{6}{c}{\multirow{3}[2]{*}{Out of memory}} \\
			+ ReLearner   & \textbf{19.52} & \textbf{32.53} & \textbf{8.11} & \textbf{11.72} & \textbf{17.96} & \textbf{16.21} & \multicolumn{3}{c}{} & \multicolumn{3}{c}{} \\
			\cellcolor{lv}$\Delta$  & \cellcolor{lv}+1.26\% & \cellcolor{lv}+1.66\% & \cellcolor{lv}+3.45\% & \cellcolor{lv}+9.71\% & \cellcolor{lv}+11.78\% & \cellcolor{lv}+6.08\% & \multicolumn{3}{c}{} & \multicolumn{3}{c}{} \\
			\bottomrule
	\end{tabular}}%
	\label{traffic-prefo}%
\end{table*}%

\begin{figure*}[!t]
    \centering
    \includegraphics[width=1\linewidth]{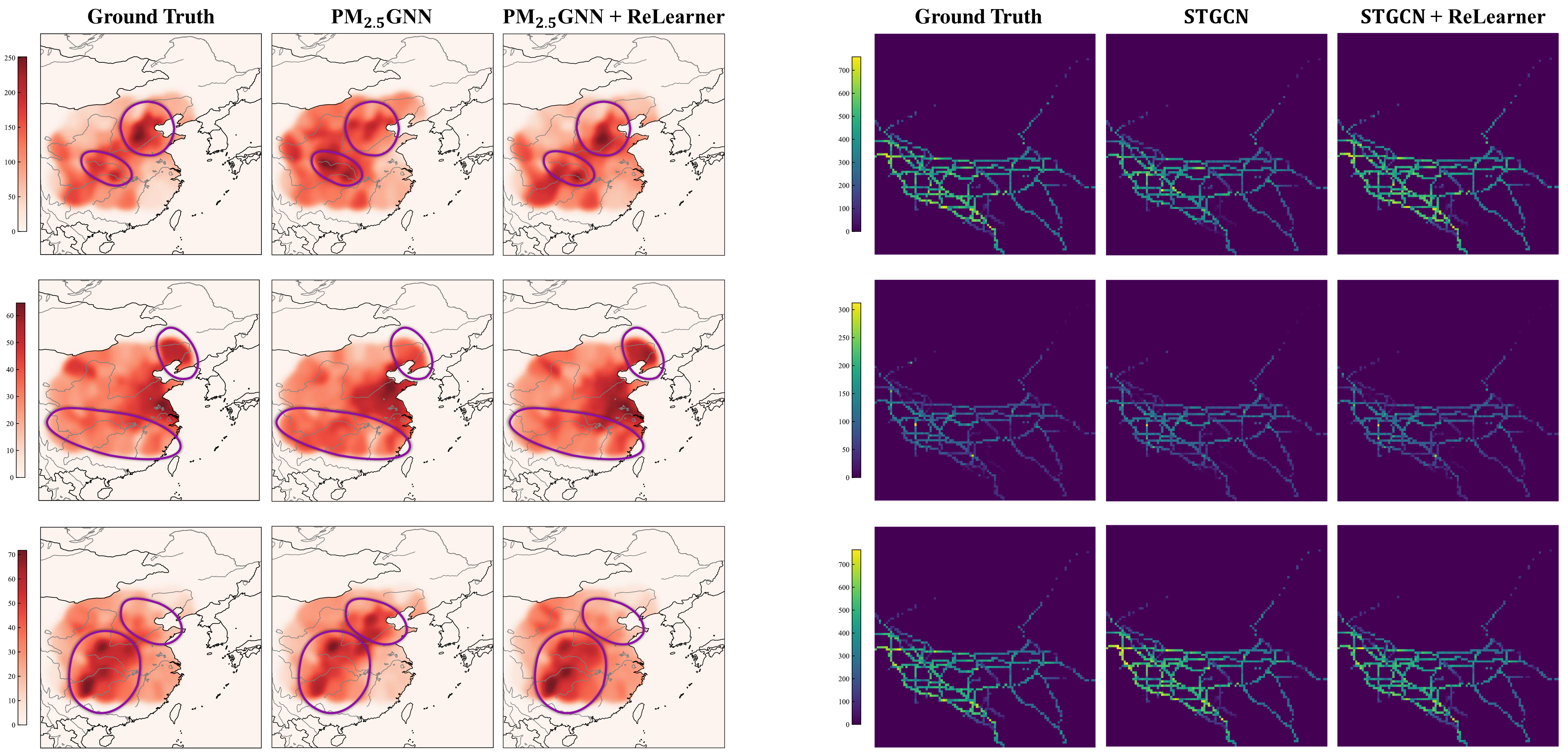} 
    \caption{Visualization of prediction comparison between PM$_{2.5}$GNN and it with ReLearner on KnowAir (Left) and LargeST-GLA (Right).}
    \label{kapredvis}
\end{figure*}

\subsection{Analysis of Predictive Performance (RQ1)}
We present the results in three tables: Table~\ref{results_traffic} reports performance on two representative datasets-LargeST-SD and LargeST-GBA datasets; Table~\ref{results_air} summarizes results on the atmospheric KnowAir dataset.
Table~\ref{traffic-prefo} provides additional comparisons on both small- and large-scale traffic datasets, including the PeMS series, METR-LA, PEMS3-Stream, and LargeST-CA. Commonly, integrating ReLearner consistently improves the predictive accuracy and robustness of various STNN backbones.

As shown in Table~\ref{results_traffic}, stronger baselines such as AGCRN and STID obtain additional performance gains, while weaker architectures such as ASTGCN benefit substantially—achieving improvements of up to 15\% on LargeST-SD. Even highly complexity models, including D$^2$STGNN and DDGCRN, exhibit consistent enhancement, indicating that ReLearner complements rather than interferes with sophisticated backbone designs. Moreover, Table~\ref{results_traffic} further demonstrates that ReLearner maintains strong scalability across both data-scarce environments (e.g., PEMS3-Stream) and large-scale settings (e.g., LargeST-CA with 8,600 nodes), continuing to deliver stable accuracy gains and illustrating its reliability under massive spatial structures. Table~\ref{results_air} shows that ReLearner also improves specialized models developed for air-quality forecasting, including GC-LSTM and PM$_{2.5}$GNN. Evaluated using widely adopted meteorological indicators—CSI, POD, and FAR—ReLearner consistently achieves superior performance across all metrics. 

We further visualize prediction results on KnowAir and LargeST-GLA, using PM$_{2.5}$GNN and STGCN as representative backbones (Fig.~\ref{kapredvis}). On KnowAir, ReLearner helps the model better capture local variations and smooth transitions, demonstrating its ability to generalize under noisy atmospheric dynamics. On LargeST-GLA, it strengthens the response of backbones to sudden traffic fluctuations, producing predictions that align more closely with the ground truth. These visualizations highlight adaptive correction capability of ReLearner and its robustness to abrupt spatiotemporal variations.

In summary, across diverse domains, datasets, and model architectures, ReLearner consistently improves predictive accuracy and stability. These gains arise not from an increase in parameter count (as analyzed in Section~\ref{notequalin}) but from its principled modeling of spatiotemporal input-label deviation between inputs and labels, validating the effectiveness of the proposed spatiotemporal residual theory.

\subsection{Effectiveness in Handling Spatiotemporal Input–Label Deviation (RQ2)}
We assess ReLearner's effectiveness in addressing the deviation in spatial and temporal dimensions. Temporal deviation is evaluated based on samples where the increase ratio of input data mean compared to label mean exceeds 75\%. Spatial deviation is identified when the similarity of input sequences between two nodes ranks in the top 20\%, while their predicted label similarity falls within the lowest 5\%. Experimental results on LargeST-SD dataset in Table \ref{ti22} demonstrate that our model enhances STNNs' modeling capacity for spatiotemporal input-labels deviation. Prediction visualizations are presented in Fig.~\ref{fig_ikeshihua}. Although the model improves node distinctiveness representation in STID through node embeddings, inefficient utilization of label features hampers comprehensive modeling of  input-label deviation. 

\begin{table}[!htbp]
  \centering
  \caption{Temporal and spatial deviation modeling.}
  \resizebox{\linewidth}{!}{
    \begin{tabular}{cccccccc}
    \toprule
    \multirow{2}[1]{*}{\textbf{Deviation}} & \multicolumn{3}{c}{\textbf{Temporal}} &       & \multicolumn{3}{c}{\textbf{Spatial}} \\
\cmidrule{2-4}\cmidrule{6-8}          & MAE   & RMSE  & MAPE  &       & MAE   & RMSE  & MAPE \\
    \midrule
    STGCN & 27.67 & 40.42 & 45.20 &       & 29.43 & 43.61 & 46.34 \\
    + ReLearner & \textbf{24.34} & \textbf{37.60} & \textbf{42.27} &       & \textbf{25.12} & \textbf{39.89} & \textbf{43.10} \\
    \midrule
    STID  & 27.06 & 41.40 & 43.63 &       & 25.78 & 39.15 & 28.68 \\
    + ReLearner & \textbf{20.55} & \textbf{32.16} & \textbf{33.50} &       & \textbf{23.06} & \textbf{38.22} & \textbf{24.31} \\
    \midrule
    STAEformer & 25.63 & 36.06 & 35.26 &       & 27.71 & 35.12 & 26.40 \\
    + ReLearner & \textbf{21.79} & \textbf{34.19} & \textbf{33.91} &       & \textbf{23.89} & \textbf{34.19} & \textbf{22.91} \\
    \midrule
    D$^2$STGNN & 21.09 & 33.37 & 34.64 &       & 24.31 & 34.62 & 24.71 \\ 
    + ReLearner & \textbf{20.31} & \textbf{31.37} & \textbf{32.85} &       & \textbf{21.09} & \textbf{33.74} & \textbf{21.13} \\
    \bottomrule
    \end{tabular}
   }%
  \label{ti22}%
\end{table}%

\begin{figure}[!ht]
\centering
\includegraphics[width=0.48\textwidth]{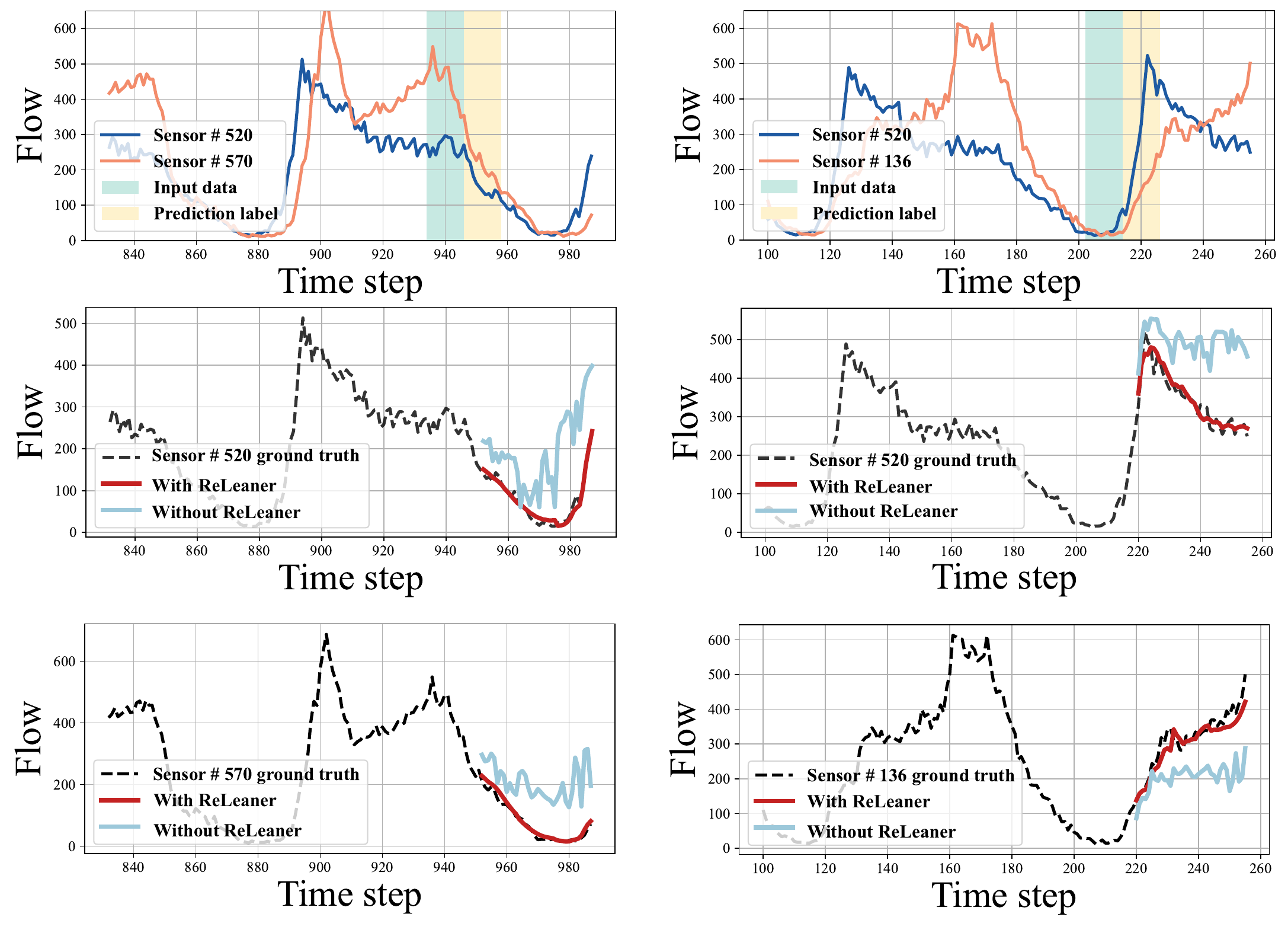}
\caption{Visualization cases of input-label deviation.}
\label{fig_ikeshihua}
\end{figure}

We evaluate the effectiveness of the model on samples with deviation in spatial and temporal dimensions. We include additional time-inconsistent samples by selecting instances where the Surge/Plummet ratio of the mean of input data relative to the mean of label data ranged from 25\% to 75\%. As shown in Table ~\ref{spatial_in} of the experimental results, we found that our proposed modules effectively enhance the modeling capability of STNN for spatiotemporal deviation features. For STID and D$^2$STGNN, the proposed modules explicitly utilize label features to mitigate the negative impact of these deviation.

\begin{table}[!htbp]
  \centering
  \caption{Modeling preformance of temporal input-label deviation with different change ratio. We use LargeST-SD dataset as example and train ReLearner with STNNs in a join-training manner.}
    \resizebox{\linewidth}{!}{ \begin{tabular}{ccccc|ccc|ccc}
    \toprule
    \multicolumn{2}{c}{\multirow{2}[0]{*}{\textbf{Surge}}} & \multicolumn{3}{c}{\textbf{25\% - 50\%}} & \multicolumn{3}{c}{\textbf{50\% - 75\%}} & \multicolumn{3}{c}{\textbf{75\% - 100\%}}  \\
\cmidrule{3-11}    \multicolumn{2}{c}{} & MAE   & RMSE  & MAPE  & MAE   & RMSE  & MAPE  & MAE   & RMSE  & MAPE \\
    \midrule
    \multirow{3}[2]{*}{STGCN} & -     & 18.12 & 46.81 & 12.10 & 23.80 & 36.83 & 13.65 & 27.67 & 40.42 & 45.20 \\
          & + ReLearner   & \textbf{15.07} & \textbf{29.82} & \textbf{10.32} & \textbf{18.13} & \textbf{27.55} & \textbf{13.06} & \textbf{24.34} & \textbf{37.60} & \textbf{42.27} \\
          & Imp.  & +16.83\% & +36.29\% & +14.71\% & +23.82 & +25.19\% & +4.32\% &   +12.03\%    &  +6.98\%     & +6.48\% \\
    \midrule
    \multirow{3}[2]{*}{STID} & -     & 14.82 & 29.17 & 11.78 & 19.65 & 30.24 & 13.47 & 27.06 & 41.40 & 43.63 \\
          & + ReLearner   & \textbf{12.04} & \textbf{26.12} & \textbf{9.16}  & \textbf{15.15} & \textbf{24.02} & \textbf{10.56} & \textbf{20.55} & \textbf{32.16} & \textbf{33.50} \\
          & Imp.  &  +18.76\%&+10.46\%&+22.24\%&+22.90\%&+20.57\%&+21.60\%&+24.06\%&+22.32\%&+23.21\% \\
    \midrule
    \multirow{3}[2]{*}{STAEformer} & -     & 12.79 & 25.57 & 10.21 & 15.26 & 28.61 & 11.64 & 25.63 & 36.06 & 35.26 \\
          & + ReLearner   & \textbf{11.74} & \textbf{25.38} & \textbf{9.86}  & \textbf{15.59} & \textbf{24.21} & \textbf{10.78} & \textbf{21.79} & \textbf{34.19}& \textbf{33.91} \\
          & Imp.  &  +8.21\%&+0.74\%&+3.43\%&2.16\%&+15.38\%&+7.39\%&+14.98\%&+5.19\%&+3.83\% \\
    \midrule
    \multirow{3}[2]{*}{D$^2$STGNN} & -     & 11.79 & 25.07 & 9.92  & 14.89 & 23.35 & 10.14 & 21.09 & 33.37 & 34.64 \\
          & + ReLearner   & \textbf{11.38} & \textbf{24.40} & \textbf{9.04}  & \textbf{14.25} & \textbf{22.59} & \textbf{9.98}  & \textbf{20.31} & \textbf{31.37} & \textbf{32.85} \\
          & Imp.  &    +3.48\%&+2.67\%&+8.87\%&+4.30\%&+3.25\%&+1.58\%&+3.70\%&+5.99\%&+5.17\% \\
    \midrule
    \midrule
    \multicolumn{2}{c}{\multirow{2}[0]{*}{\textbf{Plummet}}} & \multicolumn{3}{c}{\textbf{25\% - 50\%}} & \multicolumn{3}{c}{\textbf{50\% - 75\%}} & \multicolumn{3}{c}{\textbf{75\% - 100\%}} \\
\cmidrule{3-11}    \multicolumn{2}{c}{} & MAE   & RMSE  & MAPE  & MAE   & RMSE  & MAPE  & MAE   & RMSE  & MAPE \\
    \midrule
    \multirow{3}[2]{*}{STGCN} & -     & 24.07 & 56.63 & 18.31 & 28.10 & 48.14 & 20.65 & 33.25 & 47.47 & 22.61 \\
          & + ReLearner   & \textbf{19.59} & \textbf{35.74} & \textbf{13.06} & \textbf{21.03} & \textbf{36.49} & \textbf{16.81} & \textbf{22.95} & \textbf{39.08} & \textbf{19.73} \\
          & Imp.  &   +18.61\%&+36.89\%&+28.67\%&+25.16\%&+24.20\%&+18.60\%&+30.98\%&+17.67\%&+12.74\% \\
    \midrule
    \multirow{3}[2]{*}{STID} & -     & 26.70 & 53.64 & 15.09 & 27.15 & 51.77 & 20.77 & 27.54 & 47.78 & 26.23 \\
          & + ReLearner   & \textbf{14.44} & \textbf{26.78} & \textbf{9.50}  & \textbf{15.38} & \textbf{27.26} & \textbf{11.92} & \textbf{17.18} & \textbf{29.59} & \textbf{14.24} \\
          & Imp.  &   +45.92\%&+50.07\%&+37.04\%&+43.35\%&+47.34\%&+42.61\%&+37.62\%&+38.07\%&+45.71\% \\
    \midrule
    \multirow{3}[2]{*}{STAEformer} & -     & 18.12 & 37.22 & 13.25 & 19.94 & 37.92 & 26.00 & 23.40 & 39.45 & 18.62 \\
          & + ReLearner   & \textbf{17.30} & \textbf{36.17} & \textbf{10.79} & \textbf{18.48} & \textbf{36.27} & \textbf{15.59} & \textbf{20.04} & \textbf{36.41}& \textbf{16.64} \\
          & Imp.  &   +4.53\%&+2.82\%&+18.57\%&+7.32\%&+4.35\%&+40.04\%&+14.36\%&+7.71\%&+10.63\%\\
    \midrule
    \multirow{3}[2]{*}{D$^2$STGNN} & -     & 14.05 & 27.23 & 10.41 & 15.30 & 28.23 & 13.16 & 17.46 & 30.28 & 15.35 \\
          & + ReLearner   & \textbf{13.54} & \textbf{26.58} & \textbf{9.61}  & \textbf{15.11} & \textbf{27.44} & \textbf{11.89} & \textbf{17.09} & \textbf{29.78} & \textbf{14.06} \\
          & Imp.  &   +3.63\%&+2.39\%&+7.68\%&+1.24\%&+2.80\%&+9.65\%&+2.12\%&+1.65\%&+8.40\% \\
    \bottomrule
    \end{tabular}}
  \label{spatial_in}%
\end{table}%

\subsection{Impact on Training Dynamics and Convergence (RQ3)}
In our evaluation on the LargeST-SD dataset, we present the efficiency costs associated with integrating ReLearner with various advanced models. One key advantage of ReLearner is observed in accelerating the convergence speed of the models. Through visualizations, we showcase the training processes of several STNNs with ReLearner. By visualizing the model training procedures using the SD and Knowair-1 datasets with multiple spatiotemporal prediction as examples, as depicted in Fig.\ref{figC}, we observe that the integration of ReLearner leads to faster convergence, especially when employing the joint training strategy. This acceleration in convergence speed can be attributed to the improved model fitting to the data distribution by correcting prediction, thereby alleviating the learning burden caused by redundant features.

\begin{figure}[!h] 
\centering
\includegraphics[width=0.48\textwidth]{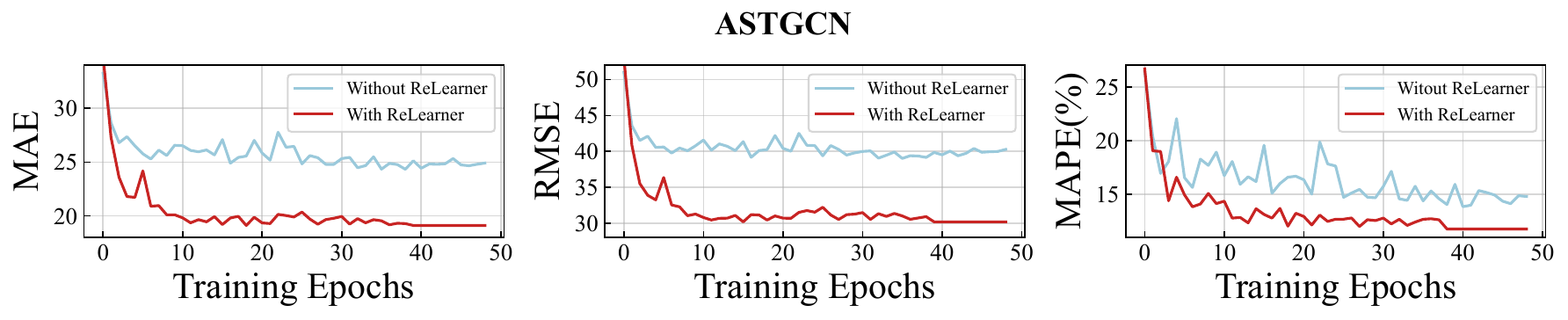}
\includegraphics[width=0.48\textwidth]{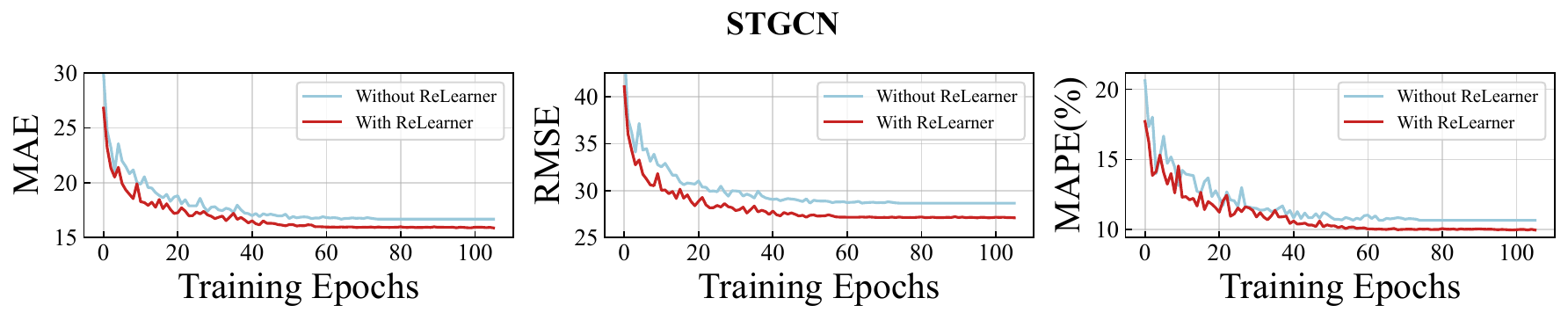}
\includegraphics[width=0.48\textwidth]{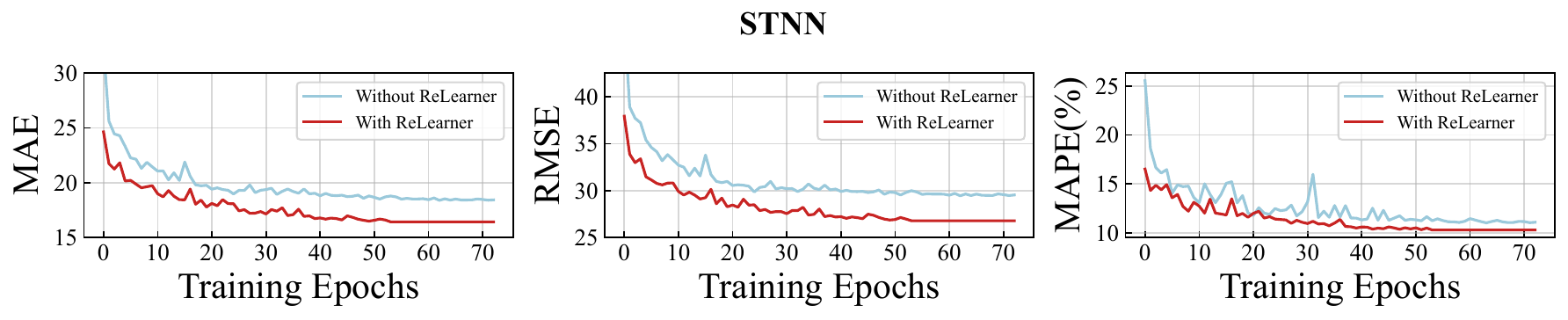}
\caption{A comparison of the convergence speed and convergence results in validation phase of baselines without ReLearner and with ReLearner on the LargeST-SD dataset.}
\label{figC}
\end{figure}

\subsection{Computational and Architecture Efficiency (RQ4)} \label{notequalin}
\noindent\ding{182} \textbf{Computational Efficiency Analysis}. Furthermore, we delve into the computational complexity of ReLearner by examining several advanced space-time prediction models, as illustrated in Table \ref{eff12}. Our analysis reveals that ReLearner offers substantial performance gains at a relatively modest model parameter cost. To optimize training efficiency, we utilize the fine-tuning training method by directly fine-tuning the pre-trained STNN with ReLearner, thereby minimizing time overhead. In conclusion, given the significant performance enhancements, the complexity burden introduced by ReLearner is considered acceptable. It is important to note that our primary focus lies on the performance rather than efficiency, which differs from the current emphasis in the spatiotemporal community.

\begin{table}[!htbp]
  \centering
  \caption{Model efficiency analysis on the LargeST-SD dataset. We report the improvement in average prediction performance over MAPE.}
    \resizebox{\linewidth}{!}{\begin{tabular}{cccccc}
    \toprule
 
    \multicolumn{2}{c}{\textbf{Method}} & \multirow{2}[0]{*}{\textbf{Parameter}} & \multirow{2}[0]{*}{\textbf{Time/epoch}(s)} & \multirow{2}[0]{*}{\textbf{Total time} (h)} & \multirow{2}[0]{*}{\textbf{Improvements}} \\
\cmidrule{1-2}    Baseline & ReLearner &       &       &       &  \\
    \midrule
    \multirow{2}[0]{*}{STGCN} & -     & 508K  & 90    & 2.0   & - \\
          & + ReLearner  & 624K  & 148   & 3.5   & +8.79\% \\
    \midrule
    \multirow{2}[0]{*}{STID} & -     & 128K  & 7     & 0.3   & - \\
          & + ReLearner  & 244K  & 12    & 0.3   & +6.06\% \\
    \midrule
    \multirow{2}[0]{*}{ASTGCN} & -     & 2.2M  & 466   & 7.3   & - \\
          & + ReLearner  & 2.6M  & 900   & 10.2  & 14.13\% \\
    \midrule
    \multirow{2}[0]{*}{AGCRN} & -     & 761K  & 365   & 7.6   & - \\
          & + ReLearner  & 1.0M  & 609   & 11.9  & 9.57\% \\
    \bottomrule
    \end{tabular}}
  \label{eff12}%
\end{table}%

\noindent\ding{183} \textbf{Architecture Efficiency Analysis}. We aim to determine whether the performance improvement achieved by ReLearner originates from its reasonable architecture innovation—the reverse spatiotemporal residual learning mechanism—rather than from a mere increase in model parameters. We use STID and D$^2$STGNN datasets as the example to demonstrate that the performance improvement brought by ReLearner does not stem from increasing parameter size or computational efforts, which can not achieve the same effect of our ReLearner, we have create a variant which stacks two STNNs (STID and D$^2$STGNN) to increase the parameter size, and we increase computational effort by training STNNs using double maximum epochs and the patience of the early stop strategy, and this variant is defined as STNN-Plus. The experimental results are shown in Table~\ref{efffff}, where we observe that simply increasing the parameter size does not lead to performance enhancement. For the complex model D$^2$STGNN, simply stacking models may actually decrease model performance, as an excessively large parameter size can lead to overfitting of the model to the data. Our performance improvement originates from the effective modeling of inconsistent features.
\begin{table}[htbp]
  \centering
  \caption{Average performance comparison of STNN, STNN-Plus, and STNN+ ReLearner on LargeST-SD dataset.}
    \begin{tabular}{cccc}
    \toprule
    \textbf{Model} & \textbf{MAE}   & \textbf{RMSE}  & \textbf{MAPE} \\
    \midrule
    STID  & 18.00 & 30.75 & 12.05 \\
    STID-Plus & 17.94 & 30.43 & 12.13 \\
    STID+ ReLearner & 17.08    & 28.92 & 11.32 \\
    \midrule
    D$^2$STGNN & 17.44 & 29.58 & 12.18 \\
    D$^2$STGNN-Plus & 17.68 & 31.04 & 12.57 \\
    D$^2$STGNN+ ReLearner & 17.19 & 28.53 & 11.20 \\
    \bottomrule
    \end{tabular}%
  \label{efffff}%
\end{table}%

\subsection{Sensitivity to Hyperparameter Configurations (RQ5)}
\noindent\ding{182} \textbf{Kernel function.} We evaluate the effect of different graph kernel types on model performance, which is explained in Equation \ref{L-equa}. The definitions of these kernels are described in Section~\ref{ap_kernel}. We take STGCN and STID as examples, and the results are shown in Table~\ref{hy_kernel222}. We find that adaptive kernel function for residual propagation achieves more accurate performance for both models. The underlying reason is that it can capture a more comprehensive spatiotemporal information. 

\begin{table}[htbp]
  \centering
  \caption{Kernel function sensitivity analysis.}
  \resizebox{\linewidth}{!}{
    \begin{tabular}{lccccccc}
    \toprule
    \multicolumn{1}{c}{\multirow{2}[0]{*}{\textbf{Kernel}}} & \multicolumn{3}{c}{\textbf{STID}} &       & \multicolumn{3}{c}{\textbf{STGCN}} \\
\cmidrule{2-4}\cmidrule{6-8}          & MAE   & RMSE  & MAPE  &       & MAE   & RMSE  & MAPE \\
    \midrule
    Transition & 17.63 & 30.40 & 11.76 &       & 19.04 & 33.53 & 13.42 \\
    DoubleTransition & \underline{17.40} & \underline{29.56} & \underline{11.44} &       & \underline{18.85} & 32.99 & \underline{12.90} \\
    Adaptive & \textbf{17.12} & \textbf{28.59} & \textbf{11.19} &       & \textbf{18.56} & \textbf{32.34} & \textbf{12.68} \\
    Data-driven & 17.99 & 30.14 & 12.84 &       & 19.30 & \underline{32.94} & 13.17 \\
    \bottomrule
    \end{tabular}}%
  \label{hy_kernel222}%
\end{table}%

\noindent\ding{183} \textbf{The number of residual propagation layers $L$.} We evaluate the sensitivity of $L$ in Equation~\ref{L-equa22}. Taking STID and STGCN as examples, experiment results on the LargeST-SD dataset are reported in Fig.\ref{L-equa}. We can see that optimal values of these two models are 2 and 3, respectively. When $L$ is smaller than the optimal value, shallow residual propagation may not effectively propagate sufficient spatiotemporal information of labels. If $L$ is equal to 0, it means that we do not utilize label information, and the large prediction errors also prove the validity of ReLearner. On the other hand, when $L$ exceeds the optimal value, excessive smoothing of information may occur due to deep layers.  Particularly, when $L$ is excessively large, ReLearner may have a negative impact, potentially due to overfitting caused by increased model complexity.

\begin{figure}[h]
    \centering
    \includegraphics[width=1\linewidth]{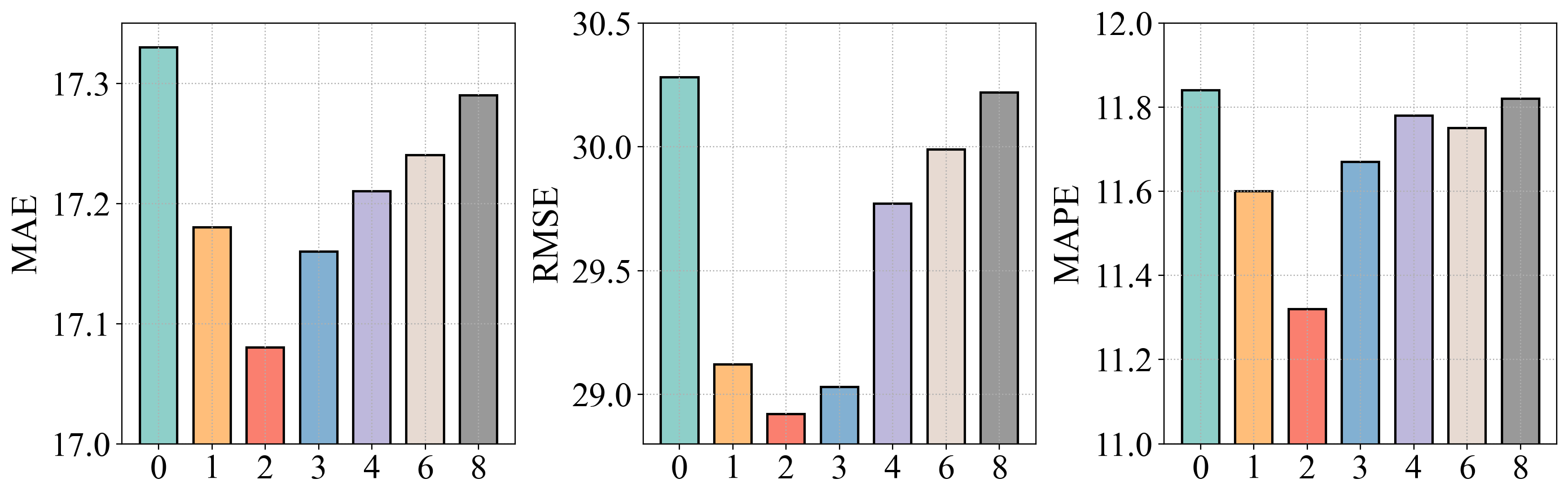}\\
    \includegraphics[width=1\linewidth]{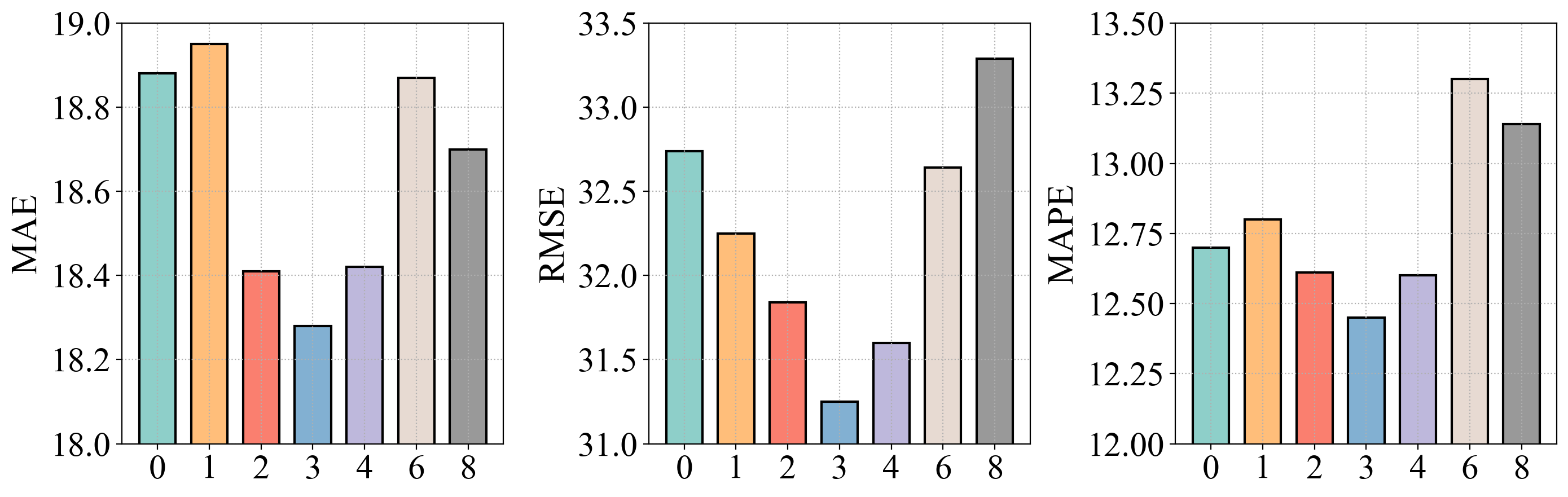}
    \caption{Hyperparameter experiment of $L$ with STID (Upper) and STGCN (lower).}
    \label{L-equa}
\end{figure}

\subsection{Comparison with Temporal Shift Technologies (RQ6)}\label{sec:temporalshift}
We compare ReLearner with current temporal shift technologies, including RevIN \cite{kim2021reversible} and Dish-TS \cite{fan2023dish}, by integrating them into STGCN and STID as the backbone models. The experimental results on the LargeST-SD dataset are presented in Table~\ref{tab:temporalshift}. We observe that each temporal shift technique performs poorly, even harming the original model's performance and capabilities. The reason lies in the current temporal shift methods' emphasis on extracting invariant statistical information from long-term time series to enhance learning, while neglecting the unique challenges posed by spatial diversity and dynamics in spatiotemporal learning. This neglect is even more pronounced when confronting the more complex challenge of input-label deviation inherent in spatiotemporal learning. In contrast, ReLearner, bolstered by spatiotemporal residual theory, effectively enhances the backbone model's performance on real-world datasets. This demonstrates that the unique challenges in spatiotemporal learning highlighted by our research motivation also represent an open research gap in the time series domain, underscoring the significance of the problems addressed by the ReLearner framework in spatiotemporal learning.

\begin{table}[!h]
  \centering
  \caption{The performance comparison between ReLearner and current temporal shift models on LargeST-SD dataset.}
  \resizebox{0.9\linewidth}{!}{  
  \begin{tabular}{clccc}
    \toprule
    \multicolumn{2}{c}{\textbf{Methods}} & \textbf{MAE} & \textbf{RMSE} & \textbf{MAPE} \\
    \midrule
    \multirow{4}[2]{*}{STID} & \multicolumn{1}{c}{-} & 18.00 & 30.75 & 12.05 \\
          & + RevIN & 25.72 & 64.92 & 21.44 \\
          & + Dish-TS & 27.75 & 54.40 & 19.85 \\
          & + \textbf{ReLearner} & \textbf{17.08} & \textbf{28.92} & \textbf{11.32} \\
    \midrule
    \multirow{4}[2]{*}{STGCN} & \multicolumn{1}{c}{-} & 19.53 & 33.79 & 13.65 \\
          & + RevIN & 21.56 & 39.18 & 16.49 \\
          & + Dish-TS & 27.59 & 45.50 & 18.71 \\
          & + \textbf{ReLearner} & \textbf{18.41} & \textbf{31.84} & \textbf{12.45} \\
    \bottomrule
    \end{tabular}
    }%
  \label{tab:temporalshift}%
\end{table}%

\subsection{Comparison with STID (RQ7)}

Moreover, we have further detailed comparisons between STID and ReLearner. \textbf{First and foremost, we emphasize that our motivation is not to surpass specific techniques within the model. Our contribution lies in introducing a general module to enhance existing spatiotemporal prediction models, which is orthogonal to existing technologies}. Just as the analysis below illustrates: the combination of STID and ReLearner outperforms other variants in terms of inconsistent feature performance, our model can synergize with other advanced technologies to generate a broader and more comprehensive impact.

\noindent\ding{182} \textbf{Performance comparison}. STID identifies spatiotemporal deviations and rectifies spatial deviation using node embedding techniques. We integrated the embedding technique from STID with ReLearner. A comparison of spatial and temporal deviation samples is presented in Table \ref{STID_TABLE3}.

Our analysis demonstrates that the proposed module excels in capturing spatiotemporal deviation features compared to node embedding techniques, particularly in addressing temporal deviation challenges. This enhanced performance is attributed to our method's explicit utilization of label information, resulting in more precise modeling. Additionally, it highlights that the node embedding technique introduced by STID may not effectively resolve spatiotemporal deviation issues. Conversely, STID + ReLearner achieves competitive performance, showcasing the potential of integrating advanced technologies with ReLearner.
\begin{table}[htbp]
  \centering
  \caption{Modeling preformance of temporal input-label deviation with different change ratio, and we use LargeST-SD dataset as example.}
    \resizebox{\linewidth}{!}{ \begin{tabular}{ccccc|ccc|ccc}
    \toprule
    \multicolumn{2}{c}{\multirow{2}[0]{*}{\textbf{Model}}} & \multicolumn{3}{c}{\textbf{25\% - 50\%}} & \multicolumn{3}{c}{\textbf{50\% - 75\%}} & \multicolumn{3}{c}{\textbf{75\% - 100\%}}  \\
\cmidrule{3-11}    \multicolumn{2}{c}{} & MAE   & RMSE  & MAPE  & MAE   & RMSE  & MAPE  & MAE   & RMSE  & MAPE \\
    \midrule
    \multirow{3}[2]{*}{STID} & -     & 14.82 & 29.17 & 11.78 & 19.65 & 30.24 & 13.47 & 27.06 & 41.40 & 43.63 \\
          & w/o em+ ReLearner   & 14.03 & 28.79 & 11.53  & 19.31 & 29.72 & 12.46 & 25.68 & 35.97 & 38.42 \\
          & STID+ ReLearner  &  12.04 & 26.12 & 9.16  & 15.15 & 24.02 & 10.56 & 20.55 & 32.16 & 33.50 \\
    \bottomrule
    \end{tabular}}
  \label{STID_TABLE3}%
\end{table}%

\begin{figure}[!ht]
\centering
\includegraphics[width=0.48\textwidth]{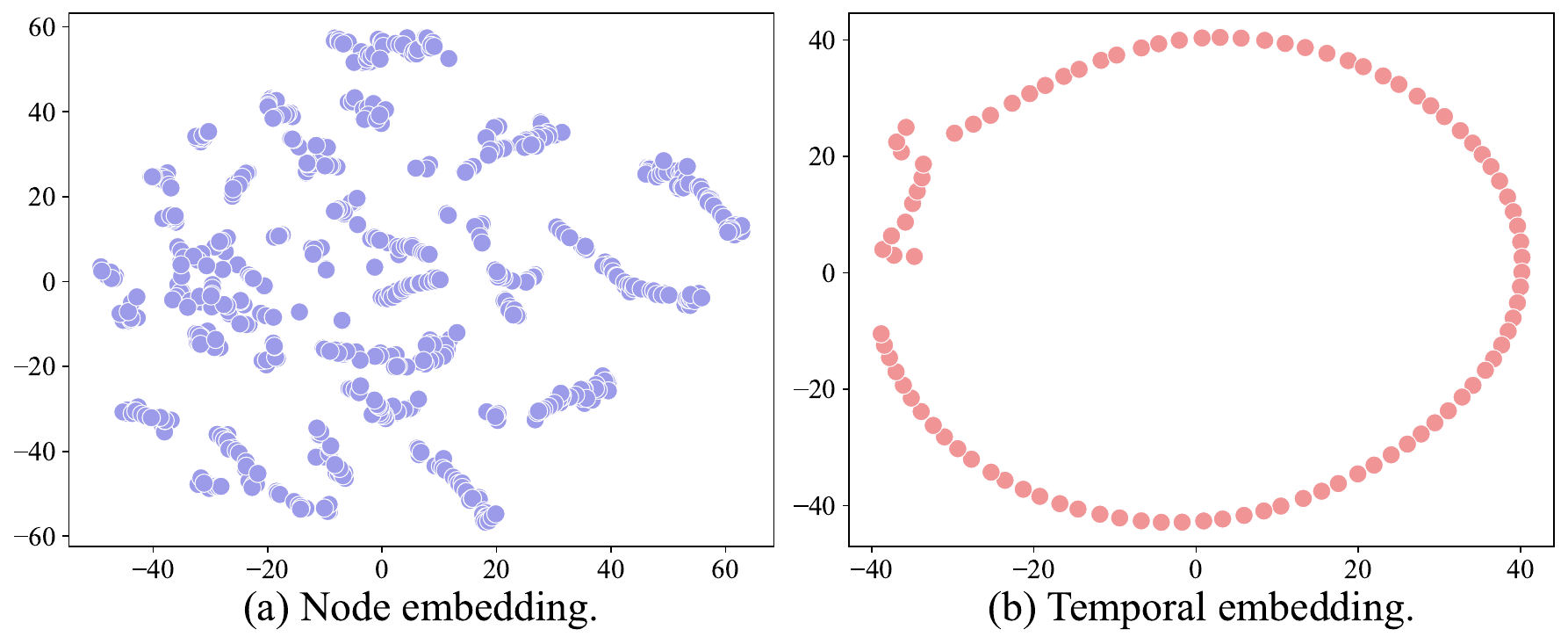}
\caption{Node embedding and temporal embedding of STID in LargeST-SD dataset.}
\label{fig_stid}
\end{figure} 

\noindent\ding{183} \textbf{Root Cause Analysis}. We further visualize the node embedding and temporal embedding of STID, as shown in Fig.\ref{fig_stid}. Regarding node embeddings, STID focuses on capturing shared patterns among nodes where nodes with similar traffic distributions cluster together. Hence, the node embeddings exhibit cluster distribution \cite{shao2022spatial}. Clearly, these shared patterns among nodes have limited utility in distinguishing spatial deviation features among nodes. Concerning time embeddings, STID captures periodic features which exhibit repetitive cycles \cite{shao2022spatial}, while temporal deviation features where traffic suddenly increases or decreases are rare. Therefore, this embedding evidently fails to capture temporal deviation.

\section{Discussion}\label{dis}
Although ReLearner demonstrates strong empirical performance and architectural flexibility, several aspects warrant further investigation. First, this study focuses primarily on four representative kernel functions for residual information propagation. Alternative kernel formulations, including those based on inter-node distributional similarity or learned kernel mixtures \cite{li2021spatial,lan2022dstagnn}, may further enhance the model’s ability to represent diverse residual dynamics. Second, while ReLearner improves predictive accuracy without significant computational overhead, future research could extend its applicability to broader spatiotemporal domains—such as climate modeling, social dynamics, and medical time series—where multi-scale dependencies and long-term temporal uncertainties are prevalent. Exploring these directions would further advance the generality, interpretability, and robustness of residual-based spatiotemporal learning.

\section{Conclusion}\label{con}
In this paper, we introduce the spatiotemporal residual theory with corresponding method \textbf{ReLearner}, a versatile and model-agnostic module designed to enhance the predictive capability of STNNs. ReLearner integrates a reverse learning procedure into the learning process through a residual-based correction mechanism, thereby enabling the model to explicitly capture and correct spatiotemporal input–label deviation. This design allows ReLearner to be seamlessly integrated into diverse STNN architectures without additional structural modifications. Extensive experiments across multiple benchmark datasets and representative STNN backbones demonstrate that ReLearner consistently improves forecasting accuracy—achieving performance gains of up to 17.36\%—while maintaining computational efficiency and scalability.

\section*{Acknowledgment}
This paper is partially supported by the National Natural Science Foundation of China (No.12227901) and the Natural Science Foundation of Jiangsu Province (BK20250482). The AI-driven experiments, simulations and model training were performed on the robotic AI-Scientist platform of Chinese Academy of Science.

\bibliographystyle{IEEEtran}
\bibliography{nips}

\begin{IEEEbiography}[{
    \includegraphics[width=0.9in,height=1.2in]{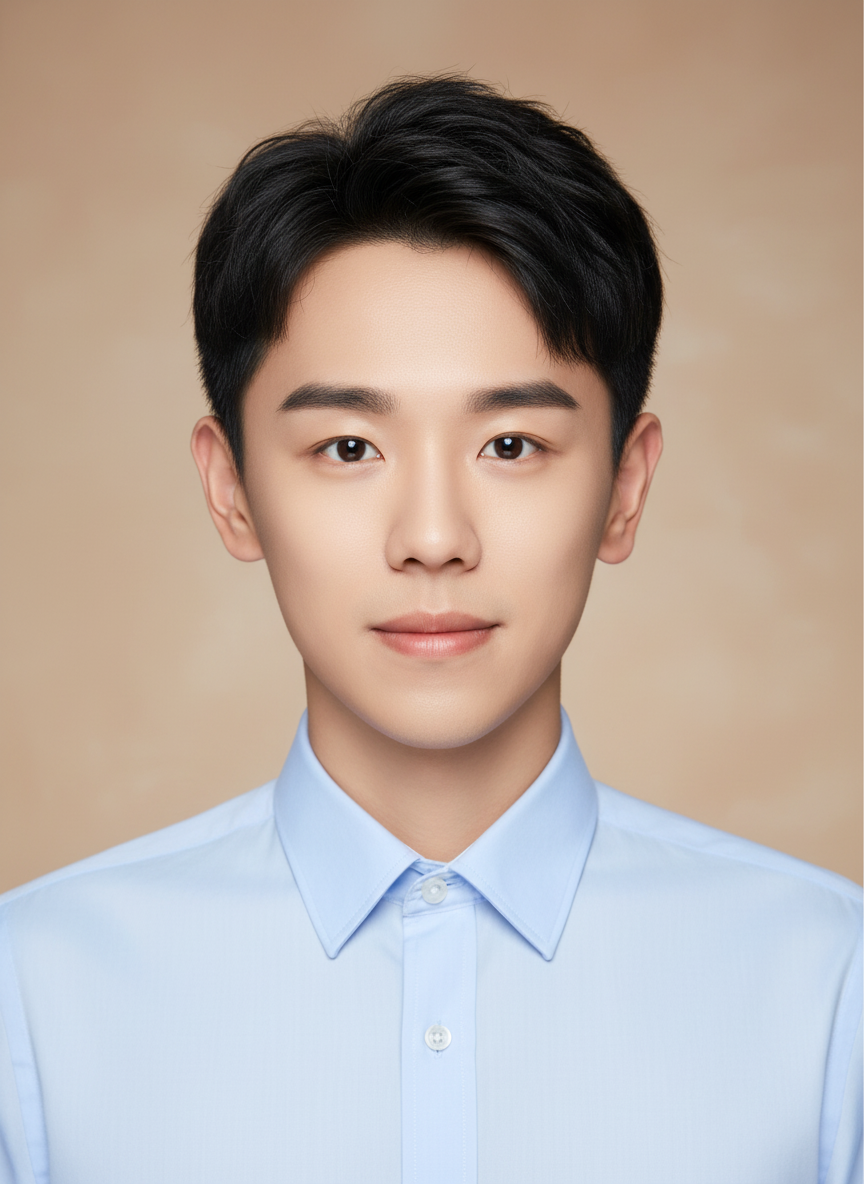}}]{Jiaming Ma} is now a Ph.D. student in the School of Artificial Intelligence and Data Science, University of Science and Technology of China (USTC). His current research focuses on deep learning structure of time series task. He has published papers in top conferences and journals such as NeruIPS, ICML, VLDB, KDD, IJCAI, and IEEE TMC.
    
    \end{IEEEbiography}
    \vspace{-0.6in}		
    
    \begin{IEEEbiography}[{ \centering \includegraphics[width=0.8in,height=1in]{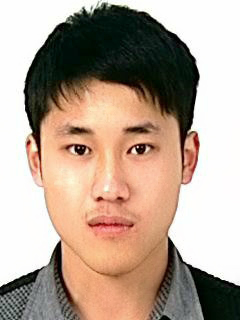}}]{Binwu Wang} is now an Associate Researcher at the University of Science and Technology of China (USTC). He got his Ph.D. degree at USTC in 2024. He has published over twenty research papers in top journals and conferences such as IEEE TMC, IEEE TITS, IEEE TVT, ICLR, SIGKDD, and AAAI. His research interests mainly include data mining, machine learning, and continuous learning. 
    \end{IEEEbiography}
    
    \vspace{-0.6in}
    
    \begin{IEEEbiography}[{\includegraphics[width=0.8in,height=1in]{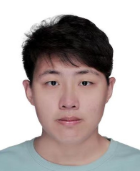}}]{Pengkun Wang} (Member, IEEE) is now an Associate Researcher at the University of Science and Technology of China (USTC). He got his Ph.D. degree at USTC in 2023, under the supervision of Professor Qi Liu and Yang Wang. His research interest mainly includes open environment machine learning, spatiotemporal data mining, and generalized AI for Science.
    \end{IEEEbiography}

    \vspace{-0.6in}
    \begin{IEEEbiography}[{\includegraphics[width=0.8in,height=1in]{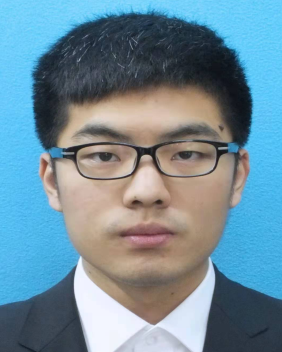}}]{Xu Wang} is now an Associate Researcher at the University of Science and Technology of China (USTC). He received his Ph.D. degree at USTC in 2023, under the supervision of Professor Zheng-Jun Zha and Yang Wang. He got his bachelor’s degree in automation at North Eastern University in 2017. His research interests mainly encompass spatiotemporal data mining, time series analysis, and the application of AI in scientific research.\end{IEEEbiography}

    \vspace{-0.6in}
    \begin{IEEEbiography}[{\includegraphics[width=0.8in,height=1in]{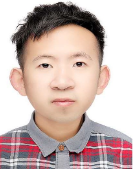}}]{Zhengyang Zhou} (Member, IEEE) is now an Associate Researcher at Suzhou Institute for Advanced Research, University of Science and Technology of China (USTC). He got his Ph.D. degree at USTC in 2023. He has published over twenty papers in top conferences and journals such as IEEE TMC, IEEE TKDE, KDD, ICLR, WWW, AAAI, NeurIPS, and ICDE. His main research interests include human-centered urban computing and mobile data mining.
    \end{IEEEbiography}
    
    \vspace{-0.6in}		
    \begin{IEEEbiography}[{\includegraphics[width=0.8in,height=1in]{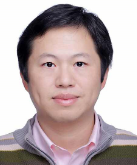}}]{Yang Wang} (Senior Member, IEEE) is now an Associate Professor at the School of Computer Science and Technology, School of Software Engineering, and School of Artificial Intelligence and Data Science at the University of Science and Technology of China (USTC). He got his Ph.D. degree at USTC in 2007. Since then, he keeps working at USTC till now as a postdoc and an associate professor successively. Meanwhile, he also serves as the vice dean of the School of Software Engineering of USTC. His research interests mainly include mobile (sensor) networks, distributed systems, data mining, and machine learning, and he is also interested in all kinds of applications of AI and data mining technologies, especially in urban computing and AI4Science. His work has been published in top-tier conferences and journals like ICLR, NeurIPS, ICML, KDD, AAAI, WWW, IEEE TKDE, IEEE TMC, and IEEE TPAMI, with over fifty papers as the first author or corresponding author. 
    \end{IEEEbiography}

\end{document}